\documentclass[10pt,twocolumn,letterpaper]{article}

\usepackage{cvpr}

\usepackage{url}
\usepackage{amsthm}
\usepackage{algorithm}
\usepackage{algorithmic}

\usepackage{amsmath}
\newtheorem{claim}{Claim}
\usepackage{adjustbox}
\usepackage{booktabs}
\usepackage{multirow}
\usepackage{array}
\usepackage{makecell}
\usepackage{float}
\newcommand{\ours}{\textsf{PRO}\xspace}
\usepackage{symbols}

\usepackage{amsmath,amsfonts,bm}

\def\1{\bm{1}}

\def\sp{space}

\def\rvx{{\mathbf{x}}}

\def\rvz{{\mathbf{z}}}

\def\vy{{\bm{y}}}

\DeclareMathAlphabet{\mathsfit}{\encodingdefault}{\sfdefault}{m}{sl}
\SetMathAlphabet{\mathsfit}{bold}{\encodingdefault}{\sfdefault}{bx}{n}

\def\gS{{\mathcal{S}}}

\def\sR{{\mathbb{R}}}

\definecolor{cvprblue}{rgb}{0.21,0.49,0.74}
\definecolor{lightcarminepink}{rgb}{0.9, 0.4, 0.38}
\usepackage[pagebackref,breaklinks,colorlinks,citecolor=cvprblue,linkcolor=lightcarminepink]{hyperref}

\def\paperID{5912} %
\def\confName{CVPR}
\def\confYear{2025}

\title{Leveraging Perturbation Robustness to Enhance Out-of-Distribution Detection}

\author{
Wenxi Chen\textsuperscript{1}, Raymond A. Yeh$^{2,\dagger}$, Shaoshuai Mou\textsuperscript{3}, Yan Gu$^{1,\dagger}$ \\
\textsuperscript{1}School of ME, \textsuperscript{2}Department of CS,
\textsuperscript{3}School of AAE, \\
Purdue University \\
{\tt\small \{chen4803, rayyeh, mous, yangu\}@purdue.edu}
}

\begin{document}
\maketitle

\begin{abstract}
Out-of-distribution (OOD) detection is the task of identifying inputs that deviate from the training data distribution. This capability is essential for safely deploying deep computer vision models in open-world environments. In this work, we propose a post-hoc method, Perturbation-Rectified OOD detection (\ours), based on the insight that prediction confidence for OOD inputs is more susceptible to reduction under perturbation than in-distribution (IND) inputs. Based on the observation, we propose an adversarial score function that searches for the local minimum scores near the original inputs by applying gradient descent. This procedure enhances the separability between IND and OOD samples. Importantly, the approach improves OOD detection performance without complex modifications to the underlying model architectures. We conduct extensive experiments using the OpenOOD benchmark~\cite{yang2022openood}. 
Our approach further pushes the limit of softmax-based OOD detection and is the leading post-hoc method for small-scale models. On a CIFAR-10 model with adversarial training,
\ours effectively detects near-OOD inputs, achieving a reduction of more than 10\% on FPR@95 compared to state-of-the-art methods.\footnote{Our code is available at~ \href{https://github.com/wenxichen2746/Perturbation-Rectified-OOD-Detection}{https://github.com/wenxichen2746/Perturbation-Rectified-OOD-Detection}. 
$^\dagger$indicates co-senior authorship.
}

\end{abstract}

\begin{figure}[t]
  \vspace{-0.6cm}
  \includegraphics[width=0.47\textwidth]{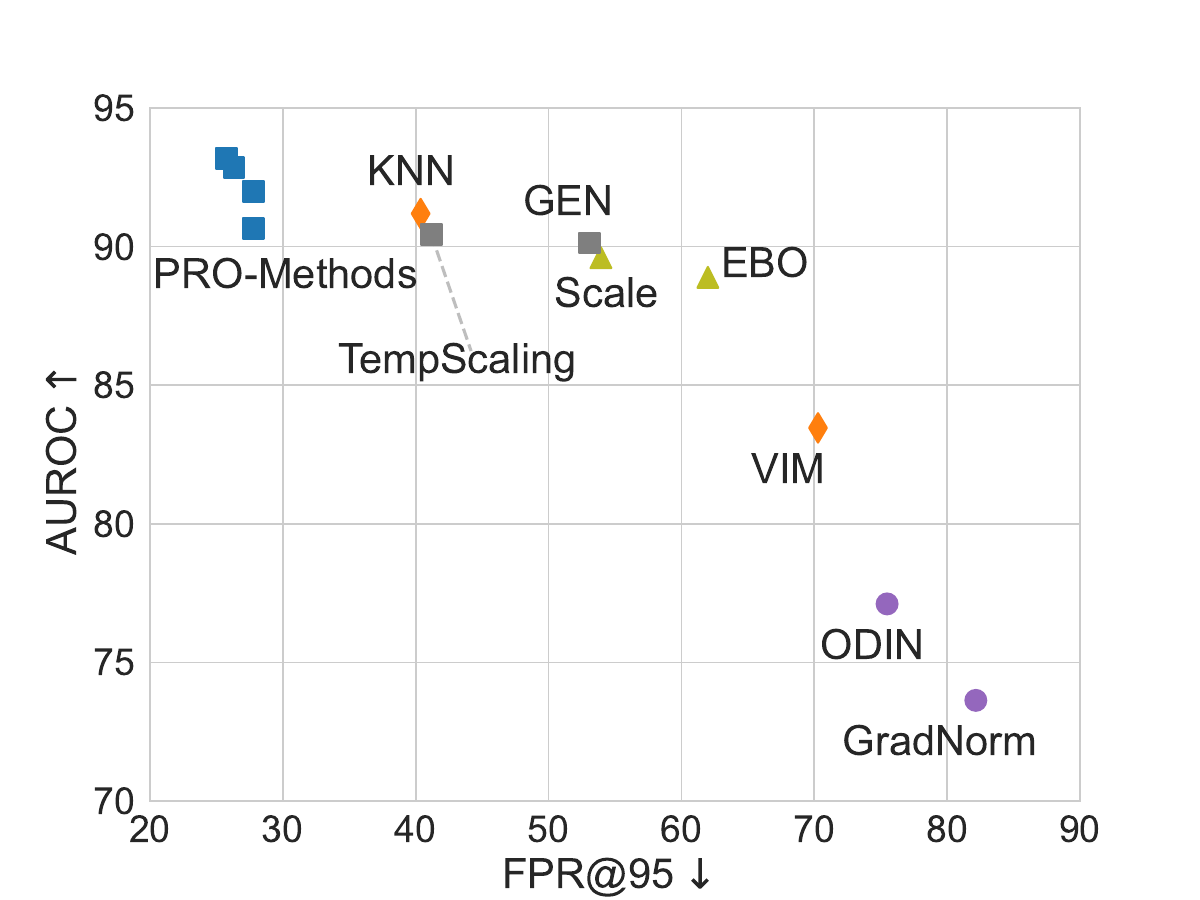}
  \centering
  \vspace{-0.2cm}
  \caption{Near-OOD detection performance tested on CIFAR-10 robust model \cite{diffenderfer2021winning}. Near-OOD includes CIFAR-100 \cite{krizhevsky2009cifar} and Tiny-ImageNet \cite{le2015tiny}. Different markers distinguish the following baseline categories: feature-based methods, such as VIM \cite{wang2022vim} and KNN \cite{sun2022out} ($\lozenge$); energy~\cite{liu2020energy} and activation modification methods, such as Scale \cite{xu2024scaling} ($\triangle$); gradient-based methods, such as ODIN \cite{liang2018enhancing} and GradNorm \cite{huang2021importance} ($\circ$); and softmax-based scores ($\square$). We apply \ours on MSP, Entropy \cite{hendrycks2016baseline}, Temperature Scaling \cite{guo2017calibration}, and GEN \cite{liu2023gen} forming four \ours methods. Notably, the proposed \ours preprocessing significantly enhances the performance of softmax scores in distinguishing challenging near-OOD data. 
 }
  \label{fig:teaser}
    \vspace{-0.3cm}
\end{figure}
\section{Introduction}
Deploying deep learning models in open-world environments presents the challenge of handling inputs that deviate from the training data. Out-of-distribution (OOD) inputs, which differ significantly from training data, often lead to incorrect predictions.
This occurs because a trained neural network cannot reliably classify inputs from unseen categories. OOD detection aims to identify such anomalous inputs, allowing fallback solutions such as human intervention~\cite{yang2024generalized}.
In-distribution (IND) data may also be affected by noise, sensor malfunctions, or adversarial attacks~\cite{croce2021robustbench}.
To address these challenges,
ongoing research focuses on improving OOD detection methods and enhancing model robustness.
Furthermore, prior studies have established connections between OOD detection and adversarial robustness~\cite{lee2018simple,bitterwolf2020certifiably,karunanayake2024out,meinke2022provably,chen2021atom}. \cite{lee2018simple} proposed a framework for detecting both OOD samples and adversarial attacks. \cite{bitterwolf2020certifiably,meinke2022provably} demonstrate that adversarial attacks can manipulate OOD samples to mislead OOD detectors. In this work, we introduce a novel OOD detection approach leveraging the robustness strength of adversarially pre-trained models.

\begin{figure*}[ht]
    \centering
    \begin{subfigure}[b]{0.53\textwidth}
        \centering
        \includegraphics[width=\textwidth]{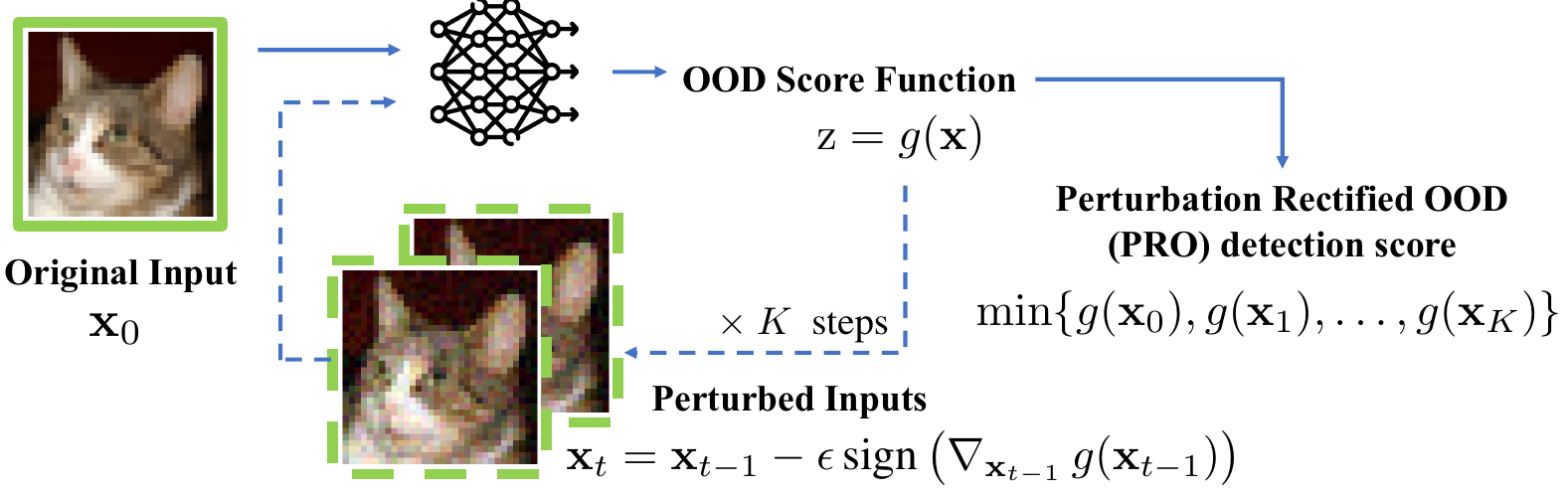} 
        \caption{Algorithm pipeline}
        \label{fig: pipelineplot1}
    \end{subfigure}
    \hspace{0.2cm}
    \begin{subfigure}[b]{0.44\textwidth}
        \centering
        \includegraphics[width=\textwidth]{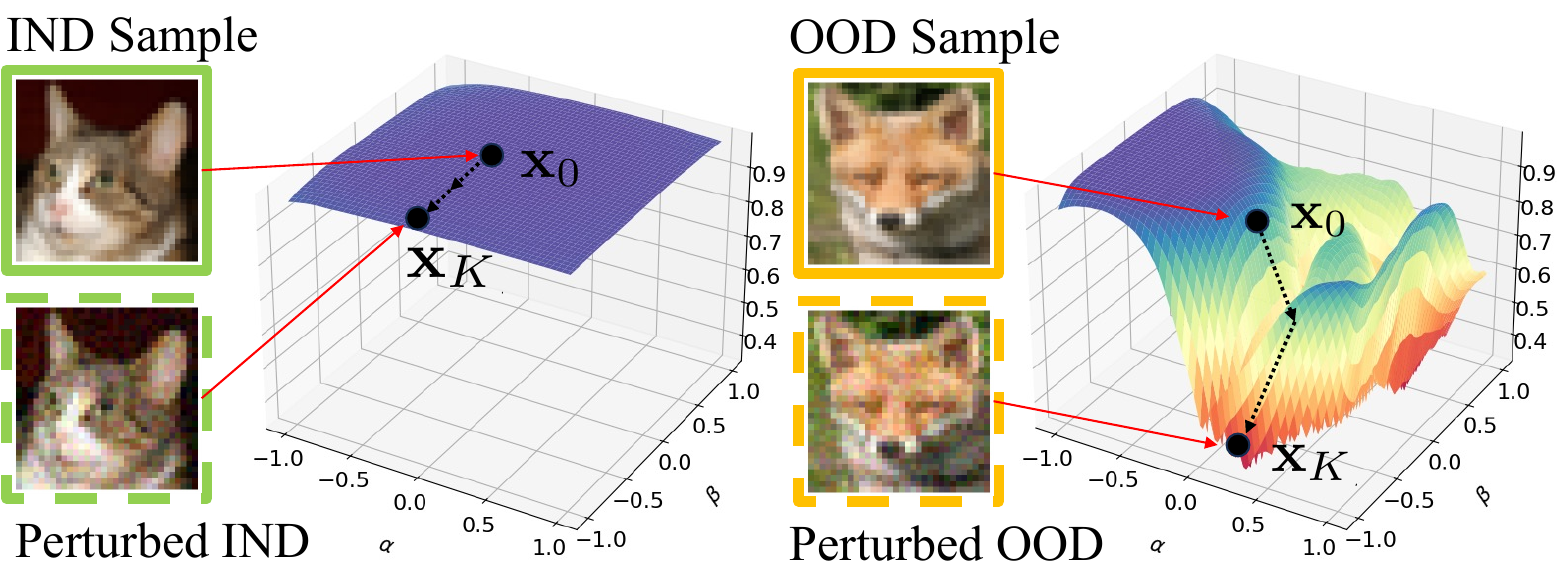}   
        \caption{MSP score for an OOD input is minimized by perturbation.}
        \label{fig: pipelineplot2}
    \end{subfigure}
    \vspace{-0.2cm}
    \caption{Algorithm overview for the proposed Perturbation Rectified OOD (PRO) detection. (a) We conduct multi-step projected gradient descent on the input image during inference to minimize the OOD detection score function. Since the score for OOD data is expected to be more vulnerable to shifts under perturbations than IND data, this process enhances the separability between IND and OOD scores. (b) MSP score landscapes for two IND and OOD samples visualized by random projection~\cite{li2018visualizing}, more examples are provided in \figref{figure:contoursplots}.}
    \vspace{-0.3cm}
    \label{fig:main}
\end{figure*}

Various OOD detection methods for image classification have emerged since the baseline method of Maximum Softmax Probability (MSP) was introduced~\cite{hendrycks2016baseline}. One line of research involves using gradient information for data preprocessing, such as ODIN~\cite{liang2018enhancing}, G-ODIN~\cite{hsu2020generalized}, and MDS with preprocessing~\cite{lee2018simple}. These works apply gradient-based perturbations to inputs to enhance prediction confidence. ODIN shows empirical differences in gradient expectations between IND and OOD data. However, evaluations from the OpenOOD benchmark~\cite{yang2022openood} reveal that ODIN reduces MSP performance across various tasks. This reduction happens because ODIN’s preprocessing tends to increase false confidence for near-OOD data, while keeping high IND confidence unaltered. This limits its ability to capture gradient differences in perturbed confidence scores.

{\bf\noindent Motivation.} Unlike previous gradient-based methods, our work builds on the observation that OOD confidence scores are more susceptible to reductions under perturbation than IND scores. We refer to this difference in sensitivity to perturbations between IND and OOD inputs as the perturbation robustness difference. Conceptually, robustness here is related to the Lipschitz constant that describes the flatness of the score function in input space. Under the same perturbation bound, OOD scores experience greater attenuation than IND scores, making them more separable. This insight suggests that adversarially robust models may be used to enhance OOD detection accuracy. Thus, we introduce a new post-hoc OOD method that leverages the model robustness towards corrupted IND inputs.

{\bf\noindent Our method.}
We propose Perturbation Rectified OOD detection (\ours) that can be incorporated with softmax-probability-based OOD detection methods to improve performance. 
By applying perturbations as a preprocessing step, \ours significantly lowers the confidence scores for OOD inputs relative to IND inputs, thereby increasing the separability between IND and OOD scores.

We evaluate \ours using the comprehensive OpenOOD benchmark~\cite{yang2022openood} across various pre-trained robust DNN backbones on CIFAR \cite{krizhevsky2009cifar} and ImageNet \cite{deng2009imagenet}. Additionally, we test leading robust models from RobustBench \cite{croce2021robustbench} to examine the synergy between OOD detection and adversarial/corruption robustness--two complementary areas critical for the safe deployment of deep learning models.
On small-scale models including CIFAR-10 and CIFAR-100~\cite{krizhevsky2009cifar}, \ours achieves leading OOD detection accuracy compared to %
existing state-of-the-art methods from benchmark~\cite{yang2022openood} %
which include IND-feature-based methods, such as, VIM~\cite{wang2022vim} and KNN~\cite{sun2022out}, activation modification methods such as Activation Shaping (ASH)~\cite{djurisic2023extremely}, and Scale~\cite{xu2024scaling}.
Furthermore, \ours works effectively in distinguishing near-OOD data, which is a substantially more challenging setting~\cite{zhang2023openood}. %
Shown in~\figref{fig:teaser}, \ours achieves top performance in near-OOD detection, excelling in both AUROC and FPR@95 metrics.
{\noindent\bf Our contributions are as follows:}
\begin{itemize}
\item We propose an
adversarial score function for OOD detection, based on the observation that IND confidence scores are more robust to perturbations than OOD inputs. See~\figref{fig:main} for an overview. We provide analysis and empirical validation of the observation. 
\item We leverage adversarial robustness to improve OOD detection. We evaluate the impact of adversarial training on OOD detection performance by utilizing the two most comprehensive benchmarks, OpenOOD and RobustBench. This establishes a new link between these two safety-critical areas in deep learning.
\item 
We demonstrate the effectiveness of the proposed \ours method as a simple, post-hoc enhancement to representative softmax scores. We perform extensive validation experiments on CIFAR-10, CIFAR-100, and ImageNet conducting a comprehensive comparison with various categories of baseline methods.
\end{itemize}

\section{Related Work}

Studies on OOD detection address several safety-critical areas in deep learning, including anomaly detection, open set recognition, and semantic and covariate domain shift detection~\cite{yang2024generalized}. Existing approaches generally involve either training modifications or post-hoc analysis.
our review of existing methods focuses primarily on those evaluated in the OpenOOD benchmark \cite{yang2022openood}\cite{zhang2023openood}, a comprehensive platform that examines various model architectures and datasets, including CIFAR~\cite{krizhevsky2009cifar} and ImageNet~\cite{deng2009imagenet}.

{\bf\noindent Training-modification methods.} These techniques require additional training protocols or data for OOD detection. Experiments from benchmark \cite{zhang2023openood,croce2021robustbench} demonstrate that data augmentation methods, such as PixMix~\cite{hendrycks2022pixmix}, AugMix~\cite{hendrycks*2020augmix} and RegMixup~\cite{pinto2022using}, are beneficial for both OOD detection and adversarial robustness. 

{\bf\noindent Representative Post-hoc methods.} These methods aim to enhance OOD detection without modifying pre-trained models. One category leverages features from IND data, as demonstrated by VIM~\cite{wang2022vim} and KNN~\cite{sun2022out}, which achieve highly competitive results. Recently, approaches like Scale~\cite{xu2024scaling}, ReAct~\cite{sun2021react}, and Ash~\cite{djurisic2023extremely} have employed modifications to neural network activations to enhance energy-based scores.

{\bf\noindent Softmax-based scores.} Beyond the classic MSP baseline, prediction entropy calculated from softmax probabilities is also regarded as a universal baseline for OOD detection~\cite{hendrycks2016baseline}. Temperature scaling~\cite{sun2022out} provides a straightforward approach to calibrating model uncertainty by scaling output logits. Recently, ~\citet{liu2023gen} introduced Generalized Entropy (GEN), demonstrating the most promising results among softmax-based scores.

{\bf\noindent Gradient-based methods.}  ODIN~\cite{liang2018enhancing}, MDS~\cite{lee2018simple}, and G-ODIN~\cite{hsu2020generalized} apply gradient-based perturbations as a preprocessing step before inference to improve OOD detection performance. GradNorm~\cite{huang2021importance} and Approximate-mass~\cite{grathwohl2019your} leverage the gradient norm directly to define an OOD detection score. These approaches share a common intuition that the landscape of score function differs between IND and OOD inputs. 

\section{Preliminaries}
\label{sec:prelim}

{\bf\noindent OOD detection for image classification.}
This study addresses OOD detection for image classification. Formally, an image classifier $f$ takes an image $\rvx$ as input and outputs the unnormalized 
$\hat{\vy} \in \sR^{C}$ across $C$ classes. These classifiers are typically trained by minimizing the cross-entropy loss. During training, it is assumed that the images $\rvx$ are drawn from an \textit{in-distribution (IND)}, denoted $P_{\text{IND}}(\rvx)$. However, during open-world testing, input data may not follow $P_{\text{IND}}(\rvx)$. We refer to this alternative distribution as $P_{\text{OOD}}(\rvx)$, representing \textit{out-of-distribution (OOD)}. The goal of ODD detection is to determine whether an image $\rvx$ is sampled from the IND distribution or not.  
\paragraph{OOD detector.} The task of OOD detection is typically framed as a one-class classification problem, where the model is trained solely on IND data without exposure to OOD examples. This is usually implemented by defining an OOD score function $g(\rvx) \in \sR$, which is then thresholded to classify an input $\rvx$ as IND or OOD.
Specifically, if $g(\mathbf{x})> \tau$, the input is classified as IND; otherwise, it is considered OOD.
A classic choice for the OOD detection score is the \textit{Maximum Softmax Probability (MSP)}
\bea
\label{equation: MSP definition}
    g_{\tt MSP}(\mathbf{x}) \triangleq \max_{y \in \{1, \dots, C\}} \frac{e^{f_y(\mathbf{x}) / T}}{\sum_{y'=1}^{C} e^{f_{y'}(\mathbf{x}) / T}}.
\eea Intuitively, MSP reflects the model's prediction confidence. The higher the confidence, the more likely the input is IND data. The temperature $T$ calibrates this confidence, reducing overconfidence when 
$T$ exceeds 1.  %
    
{\bf\noindent OOD detection metrics.} 
The primary performance metrics for evaluating OOD detectors include: (a) Area Under the Receiver Operating Characteristic Curve, denoted as AUROC, and (b) False Positive Rate at a given value $q\%$ of the True Positive Rate, denoted as FPR@$q$. A common choice is FPR@95.

\section{Approach}
In this section, we introduce the proposed \ours approach for OOD detection. Building on the framework 
reviewed in~\secref{sec:prelim}, our OOD detector also relies on a detection score derived from a pre-trained neural network. However, our method includes three key innovations. First, we introduce an ``adversarial score" to enhance an established detection score $g$ in the literature. Second, we advocate for using a pre-trained model that has been trained to be robust against adversarial attacks. Finally, we provide an analysis of the proposed detector score.

\subsection{Perturbation Rectified OOD (\ours) detection}
\label{sec:score_function}

{\noindent\bf Observation.} Our proposed \ours detector is based on the observation that a score function $g$ is more robust to local additive perturbation, within an $\epsilon$, for IND data than OOD data. More formally, we can state the above observation as an inequality in expectations that 
\bea\label{eq:noise_robost}
\mathbb{E}_{\rvx\sim{P_{\text{OOD}}}(\rvx)}[\Delta \rvz (g,\rvx)] > \mathbb{E}_{\rvx\sim{P_{\text{IND}}}(\rvx)}[\Delta \rvz (g,\rvx)],
\eea
where we define the maximum change within an $\epsilon$ of the score function $g$ as
\bea
\label{equation: dz}
\Delta \rvz (g,\rvx)= \max_{ \|\delta\|_\infty \leq \epsilon}| g(\rvx)-g(\rvx+\delta)|.
\eea

{\noindent\bf Adversarial score function.} Based on the observation in~\equref{eq:noise_robost}, we propose an adversarial score function $g^\star$ that improves upon a given existing score function $g$. This adversarial score function computes the minimum $g$ value by considering all possible perturbation $\delta$ with norm less than $\epsilon$, \ie,
\begin{equation}\label{eq:g_star}
g^\star(\rvx)= \min_{\|\delta\|_\infty \leq \epsilon} g(\rvx+\delta).
\end{equation}
To provide some intuition, %
consider a best-case scenario where IND scores are not affected by perturbation, that is, $P_{\text{IND}}(g^*(\rvx))=P_{\text{IND}}(g(\rvx))$, and OOD scores expectation has been attenuated: $\mathbb{E}_{P_{\text{ood}}}[g^*(x)]<\mathbb{E}_{P_{\text{ood}}}[g(x)]$. 
In this case, the proposed $g^\star$ will be no worse than using the given detector score $g$.

\begin{figure}[t]
\vspace{-0.3cm}
\begin{minipage}{\linewidth}
\begin{algorithm}[H]
	\caption{Solving for $g^\star(\rvx)$}
        \label{alg:g}
	\begin{algorithmic}[1]
        \STATE {\bf Input:} \textit{Step length} $\epsilon$ and \textit{step number} $K$
            \STATE \textbf{Initialize} \textit{Score record} $\gS=\{\}$
		\FOR {$t=0,1,\ldots , K$}

            \STATE Run OOD detection inference $\rvz=g(\rvx_t)$
        \STATE $\gS \leftarrow \gS \cup \{\rvz\}$
            \STATE Calculate $\quad \delta= - \epsilon \ \text{sign} \left( \nabla_{{\rvx}_{t} }\, g({\rvx}_{t}) \right)$
            \STATE Apply perturbation $\rvx_{t+1}=\rvx_t+\delta$
		\ENDFOR
        \RETURN $\min \gS$
	\end{algorithmic} 
\end{algorithm}
\end{minipage}
  \vspace{-0.3cm}
\end{figure}

{\noindent\bf Solving for the adversarial score $g^\star$.} As $g$ involves a neural network, we solve~\equref{eq:g_star} using the fast gradient sign method \cite{kurakin2017adversarial}. Given an input image $\rvx_0$, we iteratively update the image by
\begin{equation}
{\rvx}_{t} =  {\rvx}_{t-1} - \epsilon \, \text{sign} \left( \nabla_{{\rvx}_{t-1} }\, g({\rvx}_{t-1}) \right).
\end{equation}
Note that as this update does not strictly decrease $g$ at each step, we further compute the minimum across all the intermediate images,~\ie,
\bea
g^*(\rvx) \approx \min \{g(\rvx_0), g(\rvx_1), \dots, g(\rvx_K) \}.
\eea
The complete algorithm is provided in Alg.~\ref{alg:g}.

\subsection{Adversarial robustness for OOD detection}
\label{Section: Adversarially robustness for OOD detection}
From our observation in~\equref{eq:noise_robost}, we further hypothesize that using an adversarially trained neural network will benefit \ours detectors. The hypothesis is based on the finding that adversarially robust networks encourage bounded $\Delta\rvz$ of IND data, which we now discuss formally.

\begin{claim}
\label{Claim: aaloss->msp bound}
Consider a model that is trained following the adversarial robustness formulation~\cite{madry2018towards,karunanayake2024out} to have bounded training loss for IND inputs, with $y$ as true label:
\bea
    \mathbb{E}_{\rvx \sim P_{\text{IND}}(\rvx)} \left[ \max_{\|\delta\|_p < \epsilon} \mathcal{L}_{CE}(f(\mathbf{\rvx}+\delta), y) \right] < \mathcal{E},
\eea
then softmax-based OOD scores, such as MSP, have a lower bound for IND inputs. 
\end{claim}
\begin{proof}
The cross-entropy loss is equivalent to the negative log-likelihood for a given one-hot ground truth label $y$. For a trained classifier, assuming the MSP score $p_{\text{max}}$ is the probability for true label, we have:
\bea
\label{equation: lowerboundlogMSP}
    \mathbb{E}_{\rvx \sim P_{\text{IND}}(\rvx)} \left[ \max_{\|\delta\|_p < \epsilon} \left(-\log p_{\text{max}}(f(\mathbf{\rvx}+\delta))\right) \right] < \mathcal{E}\\
    \implies \mathbb{E}_{\rvx \sim P_{\text{IND}}(\rvx)} \left[ \min_{\|\delta\|_p < \epsilon} \log p_{\text{max}}(f(\mathbf{\rvx}+\delta)) \right] > -\mathcal{E}.
\eea
To establish a lower bound for MSP scores under perturbation, we leverage the convexity of the exponential function and apply Jensen's inequality:
\begin{equation}
\label{equation: lowerboundMSP}
    \mathbb{E}_{\rvx\sim P_{\text{IND}}(\rvx)} \left[ \min_{\|\delta\|_p < \epsilon} p_{\text{max}}(f(\rvx+\delta)) \right] > \exp(-\mathcal{E}).
\end{equation}

In other words, a bounded adversarial training loss leads to a lower bound for the perturbed  MSP score. Similar derivation can be extended to other softmax-based scores.
We also provide the derivation for bounding prediction entropy in the appendix.
\end{proof}

Since OOD data is not encountered during model training, the model is not encouraged to be robust to such data.
In other words, OOD scores under perturbation will likely be affected by the introduced perturbation in $g^\star$. In the experiment section, we empirically examine this behavior by visualizing the empirical distribution of $g(\rvx+\delta) - g(\rvx)$ for IND and OOD input, as shown in~\figref{figure: dz-steplength}.
This visualization confirms the validity of~\equref{eq:noise_robost}.

\begin{figure}[t]
  \includegraphics[width=0.48\textwidth]{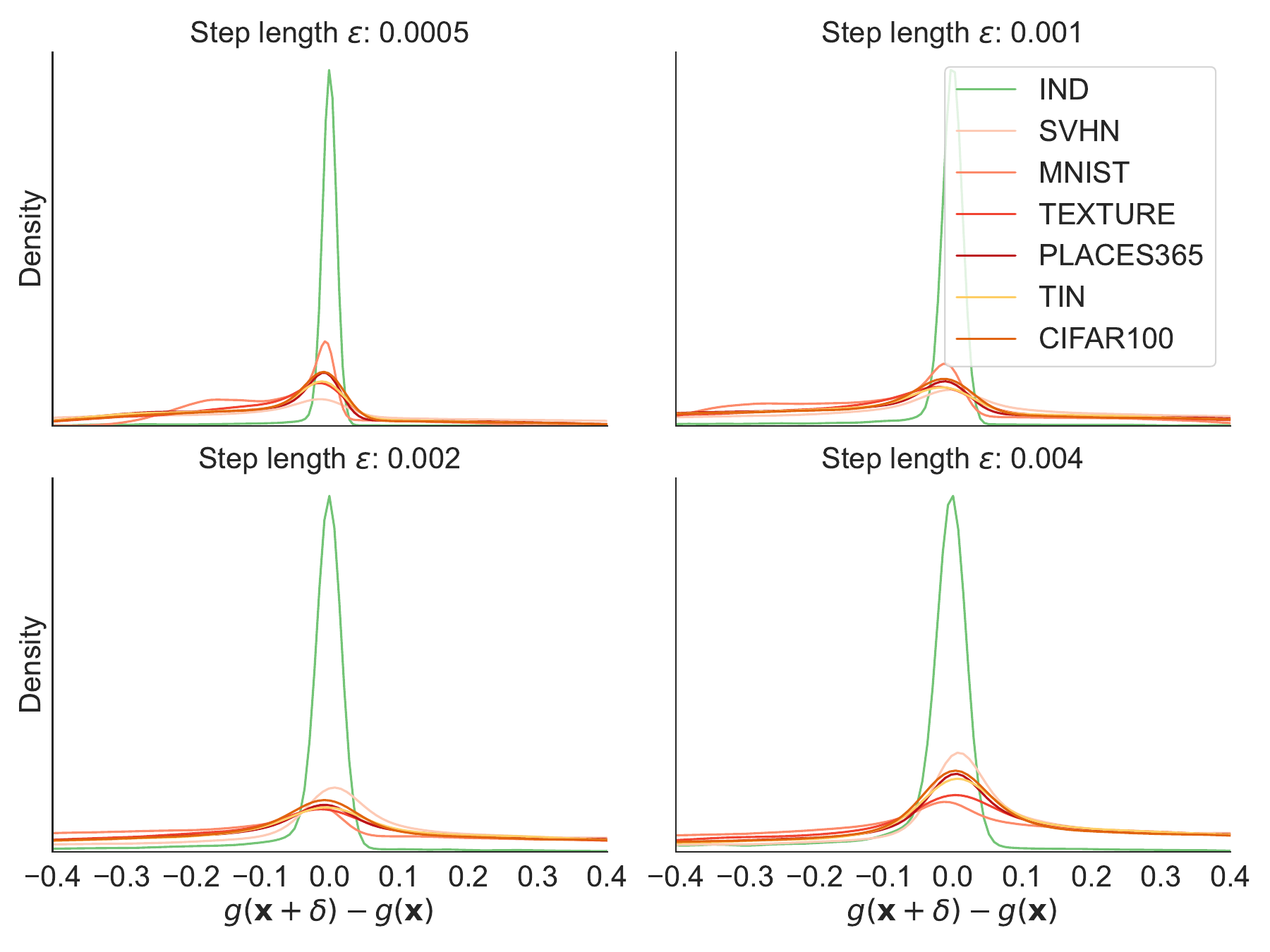}
  \centering
    \vspace{-0.75cm}
  \caption{Distribution plots of MSP score shift introduced by one-step gradient-based perturbation. OOD data endures more severe score shifts than IND data. The result is tested on a CIFAR-10 model with adversarial training~\cite{diffenderfer2021winning}.}
\vspace{-0.4cm}
  \label{figure: dz-steplength}
\end{figure}

\section{Experiment}
\begin{table*}[ht]
\centering
\resizebox{\textwidth}{!}{%
\begin{tabular}{p{0.2cm}l c c c c c c c}
\toprule
& & \multicolumn{7}{c}{\textbf{OOD detection performance: FPR@95 $ \downarrow $ / AUROC $ \uparrow $}}  \\
\cmidrule(lr){3-9}
& & \multicolumn{2}{c}{\textit{Near-OOD}} & \multicolumn{4}{c}{\textit{Far-OOD}} &  \\
\cmidrule(lr){3-4} \cmidrule(lr){5-8}
{} & Method & CIFAR100 & TIN & MNIST & SVHN & Texture & Places365 & Average \\
\midrule
\multirow{13}{*}{\rotatebox{90}{Default Model}} & MSP\cite{hendrycks2016baseline}  & 53.08/87.19 & 43.27/88.87 & 23.64/92.63 & 25.82/91.46 & 34.96/89.89 & 42.47/88.92 & 37.21/89.83 \\
& ODIN\cite{liang2017enhancing} & 77.00/82.18 & 75.38/83.55 & 23.82/\textbf{95.24} & 68.61/84.58 & 67.70/86.94 & 70.34/85.07 & 63.81/86.26 \\
& MDS~\cite{lee2018simple} & 52.81/83.59 & 46.99/84.81 & 27.30/90.10 & 25.96/91.18 & 27.94/92.69 & 47.67/84.90 & 38.11/87.88 \\
& GEN\cite{liu2023gen} & 58.75/87.21 & 48.59/89.20 & 23.00/93.83 & 28.14/91.97 & 40.74/90.14 & 47.03/89.46 & 41.04/90.30 \\
& EBO\cite{liu2020energy} & 66.60/86.36 & 56.08/88.80 & 24.99/94.32 & 35.12/91.79 & 51.82/89.47 & 54.85/89.25 & 48.24/90.00 \\
& VIM\cite{wang2022vim} & 49.19/87.75 & 40.48/89.62 & \textbf{18.35}/\underline{94.76} & \underline{19.29}/\textbf{94.50} & \textbf{21.16}/\textbf{95.15} & 41.44/89.49 & 31.65/\underline{91.88} \\
& KNN\cite{sun2022out} & \underline{37.64}/\textbf{89.73} & \textbf{30.37}/\underline{91.56} & \underline{20.05}/94.26 & 22.60/92.67 & \underline{24.06}/\underline{93.16} & \textbf{30.38}/\textbf{91.77} & \textbf{27.52}/\textbf{92.19} \\
& ASH\cite{djurisic2023extremely} & 87.31/74.11 & 86.25/76.44 & 70.00/83.16 & 83.64/73.46 & 84.59/77.45 & 77.89/79.89 & 81.61/77.42 \\
& Scale\cite{xu2024scaling} & 81.79/81.27 & 79.12/83.84 & 48.69/90.58 & 70.55/84.63 & 80.39/83.94 & 70.51/86.41 & 71.84/85.11 \\
& \cellcolor{blue!15}\ours-MSP & \cellcolor{blue!15}38.22/88.18 & \cellcolor{blue!15}32.20/90.03 & \cellcolor{blue!15}28.73/91.00 & \cellcolor{blue!15}22.34/92.35 & \cellcolor{blue!15}32.85/89.09 & \cellcolor{blue!15}33.94/89.72 & \cellcolor{blue!15}31.38/90.06 \\
& \cellcolor{blue!15}\ours-ENT & \cellcolor{blue!15}38.40/89.02 & \cellcolor{blue!15}31.64/91.00 & \cellcolor{blue!15}27.44/92.22 & \cellcolor{blue!15}21.56/93.46 & \cellcolor{blue!15}31.90/90.24 & \cellcolor{blue!15}33.12/90.73 & \cellcolor{blue!15}30.68/91.11 \\
& \cellcolor{blue!15}\ours-MSP-T & \cellcolor{blue!15}41.92/88.94 & \cellcolor{blue!15}32.63/91.31 & \cellcolor{blue!15}24.71/93.41 & \cellcolor{blue!15}20.76/93.96 & \cellcolor{blue!15}36.95/90.02 & \cellcolor{blue!15}34.20/91.22 & \cellcolor{blue!15}31.86/91.48 \\
& \cellcolor{blue!15}\ours-GEN & \cellcolor{blue!15}\textbf{37.38}/\underline{89.50} & \cellcolor{blue!15}\textbf{30.37}/\textbf{91.90} & \cellcolor{blue!15}24.07/92.91 & \cellcolor{blue!15}\textbf{19.23}/\underline{94.44} & \cellcolor{blue!15}34.91/90.27 & \cellcolor{blue!15}\underline{31.65}/\underline{91.72} & \cellcolor{blue!15}\underline{29.60}/91.79 \\
\midrule
\multirow{13}{*}{\rotatebox{90}{Robust Model: LRR ~\cite{diffenderfer2021winning}}} & MSP\cite{hendrycks2016baseline}  & 44.92/89.42 & 34.62/91.15 & 19.68/94.07 & 38.49/90.89 & 22.50/93.33 & 36.89/90.91 & 32.85/91.63 \\
& ODIN\cite{liang2017enhancing} & 75.48/77.85 & 75.48/76.37 & 26.62/95.09 & 84.96/66.60 & 66.88/82.95 & 82.98/73.76 & 68.73/78.77 \\

& MDS~\cite{lee2018simple} & 80.01/67.41 & 76.46/69.12 & 38.23/85.55 & 68.74/74.06 & 69.16/78.97 & 68.28/74.40 & 66.81/74.92 \\
& GEN\cite{liu2023gen} & 60.02/88.80 & 46.17/91.45 & 12.48/96.89 & 63.77/89.93 & 27.04/94.15 & 47.60/91.64 & 42.85/92.14 \\
& EBO\cite{liu2020energy} & 68.19/87.27 & 55.80/90.51 & \underline{9.77}/\underline{97.51} & 75.87/88.42 & 35.12/93.46 & 55.03/91.17 & 49.96/91.39 \\
& VIM\cite{wang2022vim} & 75.92/81.59 & 64.64/85.33 & 13.53/97.01 & 72.06/85.15 & 43.56/91.67 & 59.68/87.76 & 54.90/88.09 \\
& KNN\cite{sun2022out} & 45.46/90.20 & 35.28/92.18 & 16.86/95.99 & 31.48/92.85 & 22.33/94.92 & 28.81/93.49 & 30.04/93.27 \\
& ASH\cite{djurisic2023extremely} & 63.61/88.03 & 44.00/91.51 & 16.19/96.01 & 52.73/90.85 & 27.43/94.17 & 39.06/92.59 & 40.50/92.19 \\
& Scale\cite{xu2024scaling} & 59.68/88.22 & 48.21/90.97 & \textbf{8.87}/\textbf{97.71} & 71.97/88.04 & 25.93/94.62 & 51.47/91.09 & 44.35/91.77 \\
& \cellcolor{blue!15}\ours-MSP & \cellcolor{blue!15}30.92/89.82 & \cellcolor{blue!15}24.59/91.47 & \cellcolor{blue!15}27.78/91.98 & \cellcolor{blue!15}22.87/92.41 & \cellcolor{blue!15}27.13/92.32 & \cellcolor{blue!15}24.86/91.70 & \cellcolor{blue!15}26.36/91.62 \\
& \cellcolor{blue!15}\ours-ENT & \cellcolor{blue!15}31.08/91.00 & \cellcolor{blue!15}24.46/92.96 & \cellcolor{blue!15}25.74/93.65 & \cellcolor{blue!15}23.67/92.99 & \cellcolor{blue!15}24.52/93.86 & \cellcolor{blue!15}24.21/93.21 & \cellcolor{blue!15}25.61/92.95 \\
& \cellcolor{blue!15}\ours-MSP-T & \cellcolor{blue!15}\underline{30.64}/\underline{91.50} & \cellcolor{blue!15}\underline{21.99}/\underline{94.18} & \cellcolor{blue!15}13.19/96.39 & \cellcolor{blue!15}\textbf{12.64}/\underline{96.76} & \cellcolor{blue!15}\textbf{20.80}/\underline{95.01} & \cellcolor{blue!15}\underline{20.44}/\underline{94.82} & \cellcolor{blue!15}\underline{19.95}/\underline{94.78} \\
& \cellcolor{blue!15}\ours-GEN & \cellcolor{blue!15}\textbf{29.56}/\textbf{91.85} & \cellcolor{blue!15}\textbf{21.96}/\textbf{94.48} & \cellcolor{blue!15}13.20/96.44 & \cellcolor{blue!15}\underline{12.98}/\textbf{96.92} & \cellcolor{blue!15}\underline{20.86}/\textbf{95.16} & \cellcolor{blue!15}\textbf{20.39}/\textbf{95.13} & \cellcolor{blue!15}\textbf{19.82}/\textbf{95.00} \\
\bottomrule
\end{tabular}%
}
\caption{OOD detection performance with CIFAR-10 as IND. We report on the baseline model without adversarial training \cite{yang2022openood} and an adversarial robust model \cite{croce2021robustbench,diffenderfer2021winning}. Table format includes \textbf{best metric}, \underline{second best metric}, and \colorbox{blue!15}{our methods}. Observe that \ours's leading performance in distinguishing near-OOD data (\ie, CIFAR-100 and TIN), which are more challenging to detect than far-OOD data.}
\label{Table: Ciafr-10 OOD detection}
  \vspace{-0.2cm}
\end{table*}

We conduct the experiments following the evaluation protocol used in OpenOOD~\cite{yang2022openood}, a benchmark platform for OOD detection. We implemented \ours across several different OOD scores and tested it on various IND datasets.

\subsection{Experiment setup}
{\bf\noindent OOD detection methods.} To verify the generalization ability of the proposed method \ours, We implement four variants of \ours where perturbations are designed to minimize different softmax score functions. \textbf{\ours-MSP} and \textbf{\ours-MSP-T} stand for applying \ours on MSP functions without or with temperature scaling as defined in~\equref{equation: MSP definition}.
\textbf{\ours-ENT} employs the negative Shannon entropy of output softmax probabilities as the OOD detection score function. Additionally, we also apply \ours on the Generalized Entropy (GEN) \cite{liu2023gen}, which we term \textbf{\ours-GEN}. GEN also operates on softmax probability, using two additional parameters $\gamma$ and $M$: 
$
    g_{\tt GEN}(\mathbf{x})  = \sum_{j=1}^M p_j^{\gamma} (1 - p_{j})^{\gamma}.
    $

{\bf\noindent Test Datasets.}
We briefly introduce the IND datasets and corresponding OOD test sets used in the OpenOOD benchmark. Near-OOD data resemble the training data and thus are more challenging to distinguish, while far-OOD inputs are more obviously different from IND data.
\begin{itemize}
    \item CIFAR-10 model: The near-OOD datasets are CIFAR-100 and TIN \cite{le2015tiny}, while far-OOD datasets include MNIST \cite{deng2012mnist}, SVHN \cite{netzer2011reading}, Texture \cite{cimpoi2014describing}, and Places365 \cite{zhou2017places}.
    \item CIFAR-100 model: Its near-OOD dataset is CIFAR-10 and TIN, and its far-OOD datasets are the same as CIFAR-10's. 
    \item For Imagenet-1K models, near-OOD datasets include the Semantic Shift Benchmark (SSB) \cite{vaze2022openset} and NINCO \cite{bitterwolf2023or}. The far-OOD datasets consist of iNaturalist \cite{van2018inaturalist}, Texture \cite{cimpoi2014describing}, and OpenImage-O \cite{wang2022vim}.
\end{itemize}

{\bf\noindent Implementation details.} OpenOOD benchmark uses ResNet-18 and ResNet-50 \cite{he2016deep} as the backbone models for CIFAR and ImageNet, respectively. Backbone models for robust models contain WideResNet, details can be found in~\cite{croce2021robustbench}. A sample validation set is provided for methods that require hyperparameters to search for the optimal setting. The test benchmark searches for the optimal perturbation size $\epsilon$ and step number $K$  for the \ours from a hyperparameter list.

{\bf\noindent Robust models.} Since our method stems from the robustness toward perturbation, in addition to models provided in the OpenOOD benchmark, we leverage robust models from Robustbench~\cite{croce2021robustbench}, a benchmark platform for models trained against corruption or adversarial attacks. We mainly refer to the top models that are robust to general corruptions listed in RobustBench's model zoo\footnote{https://robustbench.github.io/}. %
By incorporating robust models into the OOD detection test, we intend to answer the following questions:
\begin{itemize}
    \item How would adversarial training affect models' OOD detection performance?
    \item Will adversarial training improve OOD detection performance of \ours? 
\end{itemize}

{\bf\noindent Baseline methods.}
 Classic baselines include softmax scores, such as MSP~\cite{hendrycks2016baseline}, TempScaling~\cite{guo2017calibration}, Entropy~\cite{hendrycks2016baseline}, and logits-based scores, such as MLS~\cite{hendrycks2019scaling} and EBO~\cite{liu2020energy}. ODIN~\cite{liang2017enhancing} is also a highly related baseline using perturbation-based preprocessing. GEN~\cite{liu2023gen} has been considered one of the most promising methods using softmax scores.
 
 We also consider the most competitive methods for each dataset evaluated by OpenOOD benchmark \cite{yang2022openood,zhang2023openood}. For CIFAR-10 dataset, two feature-based methods, VIM~\cite{wang2022vim} and KNN~\cite{sun2022out} have leading performance and both require IND data. As for CIFAR-100, MLS and RMDS \cite{ren2021simple} have the best AUROC performance for near- and far-OOD data, respectively. 
 A recent category of methods uses activation modification and energy scores, including Scale~\cite{xu2024scaling}, ASH~\cite{djurisic2023extremely}, and ReAct~\cite{sun2021react}. They have the most promising results on ImageNet.
 
\begin{figure}[t]
  \includegraphics[width=0.475\textwidth]{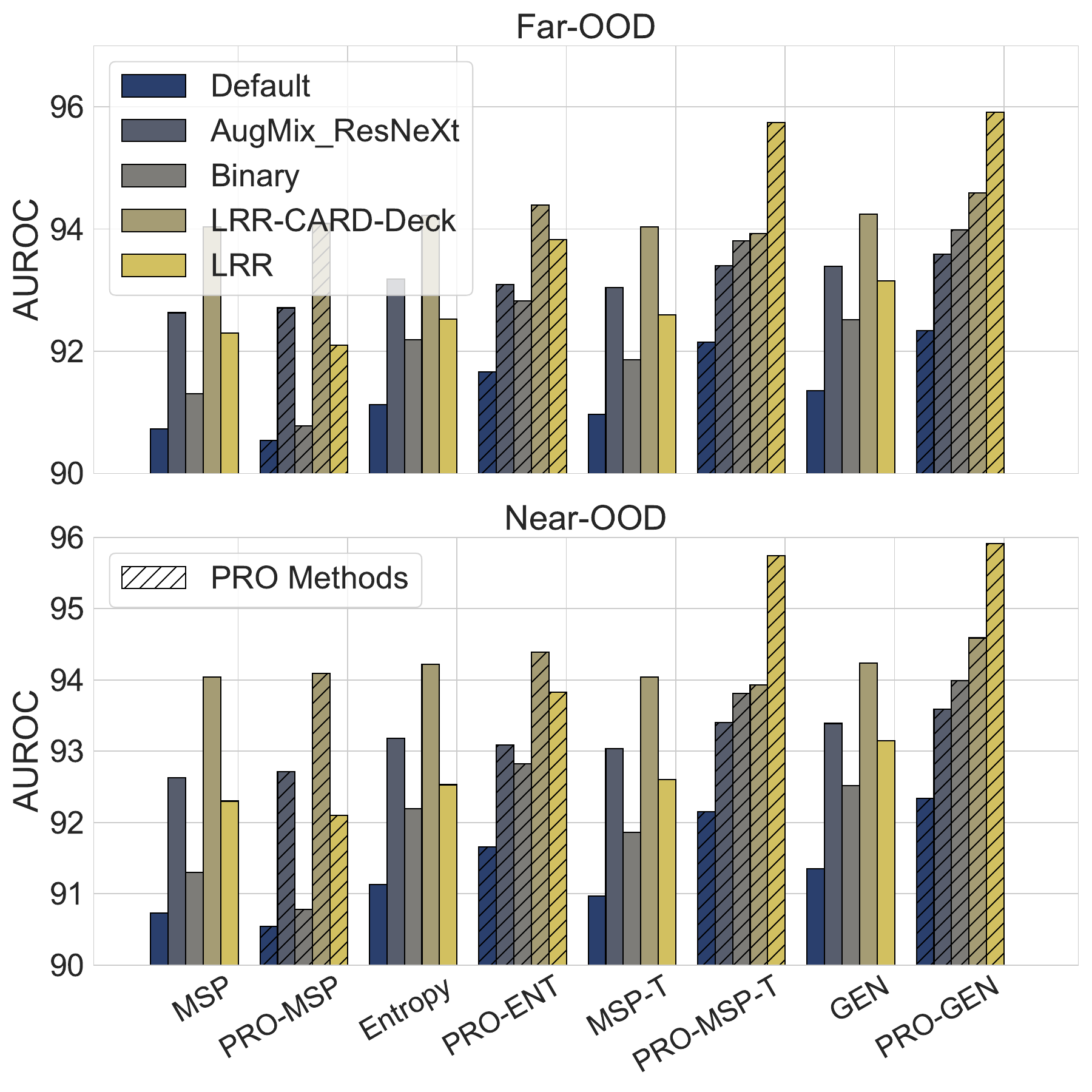}
  \centering
    \vspace{-0.8cm}
  \caption{AUROC performance on CIFAR-10 tested across baseline model \cite{yang2022openood} and adversarially robust models (\ie, AugMix\_ResNeXt, Binary, LRR-CARD-Deck, and LRR) \cite{diffenderfer2021winning,hendrycks*2020augmix}. \ours stably enhance four representative softmax scores, MSP, entropy \cite{hendrycks2016baseline}, temperature scaling MSP-T \cite{guo2017calibration}, and GEN \cite{liu2023gen}.}
  \label{figure: aurocsbar-cifar10}
    \vspace{-0.3cm}
\end{figure}

\subsection{OOD detection performance}
\subsubsection{CIFAR-10}
{\bf\noindent Robust model improves \ours performance.}
\tabref{Table: Ciafr-10 OOD detection} summarizes the OOD detection performance with CIFAR-10 as IND. We present the comparison between a default model provided by OpenOOD \cite{zhang2023openood} and an adversarial robust model from RobustBench \cite{croce2021robustbench}. The robust model is trained with Learning Rate Rewinding (LRR) \cite{diffenderfer2021winning}, which has leading robust accuracy under common corruption. The result for default model is averaged over 3 different checkpoints, while the robust model only has one checkpoint. We also present the AUROC performance tested on other models from RobustBench in \figref{figure: aurocsbar-cifar10}.  

{\bf\noindent Comparison with SOTA baselines.} Recent studies~\cite{liu2023gen,tang2024cores} have limited their comparison to IND-free, post-hoc methods, assuming IND-feature-based approaches (\eg, VIM and KNN) gain an extra advantage by using IND data or are not generally applicable. Nevertheless, we see that \ours-enhanced scores, as an IND-free technique, significantly surpass IND-feature-based baselines when tested on robust models. The results also show that \ours has top performance on distinguishing near-OOD data such as CIFAR-100 and TIN for both the default and robust models compared to all baselines.

OOD methods that use activation modification and energy scores (\eg, ReAct, Ash, and Scale) do not seem to perform well on the small-scale model CIFAR-10. Another noteworthy comparison is with ODIN, which also uses gradient-based perturbation. We can see that ODIN suffers from degraded performance compared to the original MSP score.

\begin{table}[t]
\centering
\begin{tabular}{p{0.3cm}l c c}
\toprule
& & \multicolumn{2}{c}{\textbf{FPR@95 $ \downarrow $ / AUROC $ \uparrow $}}  \\
{} & Method & Near-OOD & Far-OOD \\
\midrule
\multirow{15}{*}{\rotatebox{90}{Default Model}} & MSP\cite{hendrycks2016baseline}  & 54.80/80.27 & 58.70/77.76  \\
& Entropy\cite{hendrycks2016baseline} & 54.58/81.14 & 58.33/78.97 \\
& TempScaling\cite{guo2017calibration} & \underline{54.49}/80.90 & 57.94/78.74 \\
& GEN\cite{liu2023gen} & \textbf{54.42}/\underline{81.31} & 56.71/79.68 \\
& VIM\cite{wang2022vim} & 62.63/74.98 & \textbf{50.74}/81.70 \\
& KNN\cite{sun2022out} & 61.22/80.18 & 53.65/\underline{82.40} \\
& ODIN\cite{liang2017enhancing} & 57.91/79.90 & 58.86/79.28 \\
& EBO\cite{liu2020energy} & 55.62/80.91 & 56.59/79.77 \\
& MLS\cite{hendrycks2019scaling} & 55.47/81.05 & 56.73/79.67 \\
& RMDS\cite{ren2021simple} & 55.46/80.15 & \underline{52.81}/\textbf{82.92} \\
& Scale\cite{xu2024scaling} & 55.68/80.99 & 54.09/81.42 \\
& \cellcolor{blue!15}\ours-MSP & \cellcolor{blue!15}56.10/80.78 & \cellcolor{blue!15}58.53/78.26 \\
& \cellcolor{blue!15}\ours-ENT & \cellcolor{blue!15}55.19/81.22 & \cellcolor{blue!15}57.18/79.44 \\
& \cellcolor{blue!15}\ours-MSP-T & \cellcolor{blue!15}55.65/81.04 & \cellcolor{blue!15}55.52/79.71 \\
& \cellcolor{blue!15}\ours-GEN & \cellcolor{blue!15}54.73/\textbf{81.36} & \cellcolor{blue!15}56.13/79.81 \\
\midrule
\multirow{15}{*}{\rotatebox{90}{Robust Model: LRR-CARD-Deck\cite{diffenderfer2021winning}}} & MSP\cite{hendrycks2016baseline}  & 52.94/81.42 & 54.10/78.60  \\
& Entropy\cite{hendrycks2016baseline} & 52.94/81.85 & 54.10/79.10 \\
& TempScaling\cite{guo2017calibration} & 52.94/81.42 & 54.10/78.60 \\
& GEN\cite{liu2023gen} & 52.96/81.88 & 54.10/79.16 \\
& VIM\cite{wang2022vim} & 85.07/58.13 & 73.61/65.85 \\
& KNN\cite{sun2022out} & 69.64/72.18 & \textbf{37.41}/\textbf{87.26} \\
& ODIN\cite{liang2017enhancing} & 54.07/79.38 & 50.53/81.17 \\
& EBO\cite{liu2020energy} & 52.95/81.90 & 54.10/79.16 \\
& MLS\cite{hendrycks2019scaling} & 52.94/81.42 & 54.10/78.61 \\
& RMDS\cite{ren2021simple} & \textbf{51.13}/82.08 & \underline{49.57}/\underline{81.50} \\
& Scale\cite{xu2024scaling} & 77.39/67.26 & 58.42/78.90 \\
& \cellcolor{blue!15}\ours-MSP & \cellcolor{blue!15}52.43/82.09 & \cellcolor{blue!15}53.75/78.48 \\
& \cellcolor{blue!15}\ours-ENT & \cellcolor{blue!15}52.53/\underline{82.49} & \cellcolor{blue!15}56.29/78.17 \\
& \cellcolor{blue!15}\ours-MSP-T & \cellcolor{blue!15}53.06/81.93 & \cellcolor{blue!15}56.67/77.53 \\
& \cellcolor{blue!15}\ours-GEN & \cellcolor{blue!15}\underline{52.38}/\textbf{82.50} & \cellcolor{blue!15}55.89/78.42 \\
\bottomrule
\end{tabular}
  \vspace{-0.15cm}
\caption{OOD detector performance with CIFAR-100 as IND. We listed the averaged metrics in near-OOD and far-OOD, emphasizing \ours is the most powerful post-hoc method for distinguishing near-OOD, especially for models with adversarial training. 
}
\label{table: cifar-100}
  \vspace{-0.6cm}
\end{table}

\subsubsection{CIFAR-100}
{\bf\noindent \ours is most competitive for near-OOD detection.} We present averaged near-OOD and far-OOD performance in~\tabref{table: cifar-100}, highlighting that \ours variants generally demonstrate competitive performance in the near-OOD setting, which is emphasized by the relatively high AUROC scores. Noteworthy is that the enhancement of applying \ours to softmax scores is more substantial for the robust model. %

{\bf\noindent Comparison with ODIN.} One can also notice that ODIN tends to improve MSP in far-OOD settings but suffers from performance degradation for near-OOD, while \ours does not. Intuitively, \ours pushes OOD scores down, thus helping to separate near-OOD with falsely high prediction confidence. Meanwhile, ODIN aims to do the opposite and boost the OOD scores, making near-OOD have higher prediction confidence and harder to distinguish. %

\begin{figure}[t]
  \includegraphics[width=0.475\textwidth]{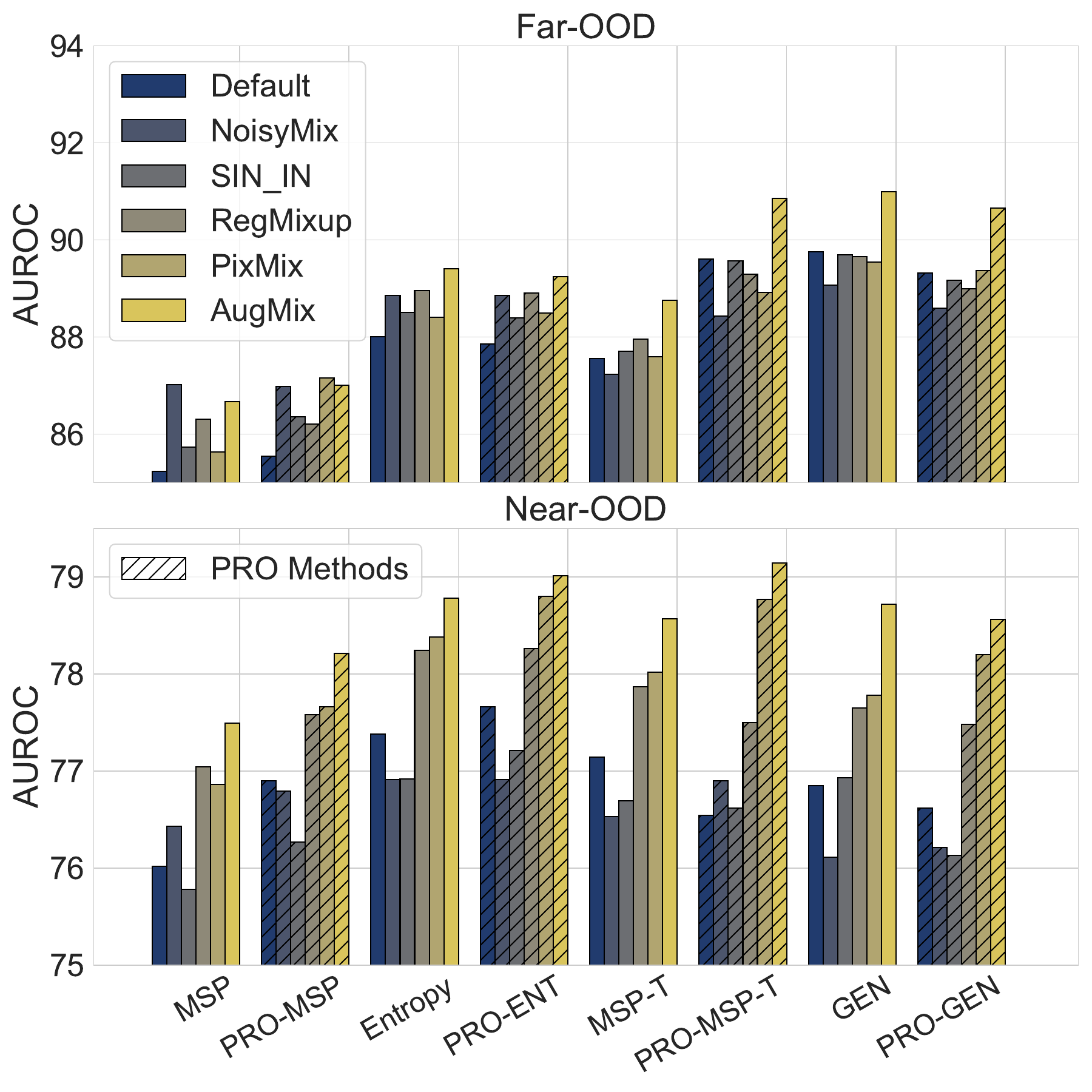}
  \centering
    \vspace{-0.7cm}
  \caption{AUROC performance of PRO methods tested on ImageNet. PRO works most well with data augmentation methods PixMix \cite{hendrycks2022pixmix} and AugMix \cite{hendrycks*2020augmix}, while the other two robust models NoisyMix \cite{erichson2024noisymix} and SIN\_IN \cite{geirhos2018imagenet} have negative impacts on OOD detections. MSP, temperature scaling, and Entropy can still benefit from PRO to enhance near-OOD detection.}
  \label{figure: aurocsbar-imagenet}
    \vspace{-0.4cm}
\end{figure}

\subsubsection{ImageNet-1K}
{\bf\noindent } Our test on ImageNet shows \ours has hindered performance as the model scale increases. Activation modification methods such as Scale~\cite{xu2024scaling}, ASH \cite{djurisic2023extremely}, and ReAct~\cite{sun2021react} work best for ImageNet, outperforming baselines from other categories. Due to the page limit, detailed OOD detection results are provided in the appendix.  

\figref{figure: aurocsbar-imagenet} illustrates the performance impact of different adversarial training protocols and data augmentation methods. PixMix~\cite{hendrycks2022pixmix} and AugMix~\cite{hendrycks*2020augmix}, as provided in the OpenOOD benchmark~\cite{zhang2023openood}, both improve model robustness and significantly enhance the AUROC result for \ours methods. Additionally, we include two adversarially robust models, NoisyMix \cite{erichson2024noisymix} and SIN-IN~\cite{geirhos2018imagenet}. However, NoisyMix and SIN-IN result in degraded performance of softmax scores, particularly in near-OOD scenarios.

The figure also compares softmax baselines with \ours methods, distinguished by the slashed texture. While \ours does not show significant improvement for far-OOD cases on ImageNet, \ours-MSP, \ours-MSP-T, and \ours-ENT exhibit AUROC gains in near-OOD detection. In the following section, we discuss how model scale affects adversarial robustness and the implications for perturbation-based OOD separation.

\begin{figure*}[ht]
  \includegraphics[width=0.99\textwidth]{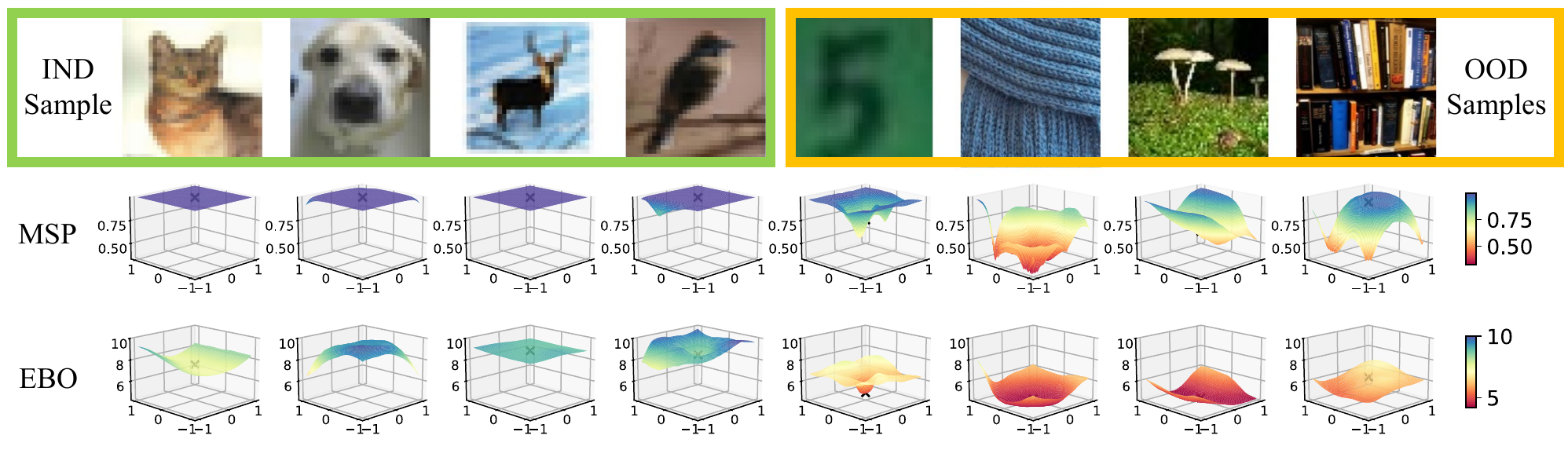}
  \centering
  \vspace{-0.3cm}
  \caption{Visualization of OOD score function landscape regarding image perturbation including maximum confidence (MSP) and energy-based OOD (EBO) detection score. We select four IND images from CIFAR-10~\cite{krizhevsky2009cifar}, and four OOD images from SVHN~\cite{netzer2011reading}, Texture~\cite{cimpoi2014describing}, TIN~\cite{le2015tiny}, and Place365~\cite{zhou2017places}, deploying random projection to plot the landscape \cite{li2018visualizing}. 
  The contour color indicates the score value, which is proportional to the contour height.
  The $x$- and $y$-axes correspond to $\alpha$ and $\beta$ in Eq.~\eqref{equation: random projection}, representing perturbation magnitudes in different directions. Scores for unperturbed images are marked with ``$\times$'' in contours.}
  \vspace{-0.3cm}
  \label{figure:contoursplots}
\end{figure*}
\subsection{Perturbation robustness analysis}
{\bf\noindent Score function landscape visualization.}
We adopt the random projection method \cite{li2018visualizing} to provide an intuitive visualization of perturbation robustness. we aim to visualize the landscape of OOD scores functions in input image space. The visualization involves two random perturbation directions $\delta_1$ and $\delta_2$. Given an image data $\mathbf{x}$, we plot the contour of function $\mathrm{z}(\alpha,\beta)$ defined as:
$
\label{equation: random projection}
    \mathrm{z}(\alpha,\beta)=g(\mathbf{x}+\alpha \delta_1 +\beta \delta_2).
$
Note that the landscape in the gradient-based direction would be much sharper compared to other random directions.

\figref{figure:contoursplots} visualizes various IND and OOD images for score function visualization as described in the caption. The smoother, less varied contour of the MSP function for IND inputs suggests greater robustness against perturbations when compared to the more varied MSP contours for OOD inputs. We observe that softmax-based scores such as MSP generally have a more stable landscape than logit-based scores, such as EBO. We hypothesize that this is due to the subtler connection between logits and the cross-entropy loss. %

{\bf\noindent Score shift distribution.}
We use the robustness metric of score shift to empirically validate the inequality in~\equref{eq:noise_robost}. \figref{figure: dz-steplength} indicates the same perturbation would induce a more significant shift for OOD inputs than for IND inputs. It is noteworthy that, under a large perturbation step, a large portion of OOD scores have been increased even when the perturbation is at a negative gradient direction. This supports the necessity of including a minimization step described in our approach to prevent the perturbation from boosting false confidence for OOD inputs.

{\bf\noindent \ours does not depend on adversarial training.} 
We observe that even baseline models without adversarial training exhibit a robustness difference between IND and OOD inputs. This occurs because standard training protocols inherently create smoother score landscapes for IND data, resulting in inherent robustness. The property suggests \ours can be adopted to enhance OOD detection performance for models without adversarial training, as indicated in~\tabref{Table: Ciafr-10 OOD detection} and~\tabref{table: cifar-100}.

{\bf\noindent Increase of model scale undermines its IND perturbation robustness.} Experimental results have shown that \ours works best with CIFAR-10, a small-scale model with a limited class number. The enhancement of the method \ours in softmax scores gradually attenuates as the model scale increases. \figref{figure: MSPshift-modelscale} provides insights on why \ours has limitations working with large-scale models. It shows the difference in score shift introduced by the same level perturbation for different model scales. 

In the left plot of~\figref{figure: MSPshift-modelscale} describing IND score shifts, the distribution centered at 0 suggests that the score is barely altered by perturbation. 
We highlight the insight that scores for IND inputs suffer from greater shifts as the model's class numbers increase. In other words, under the same training protocol, large-scale models are more vulnerable to score shift under perturbation, thus limiting the enhancement of adopting \ours methods.

\begin{figure}[t]
    \centering
  \includegraphics[width=0.475\textwidth]{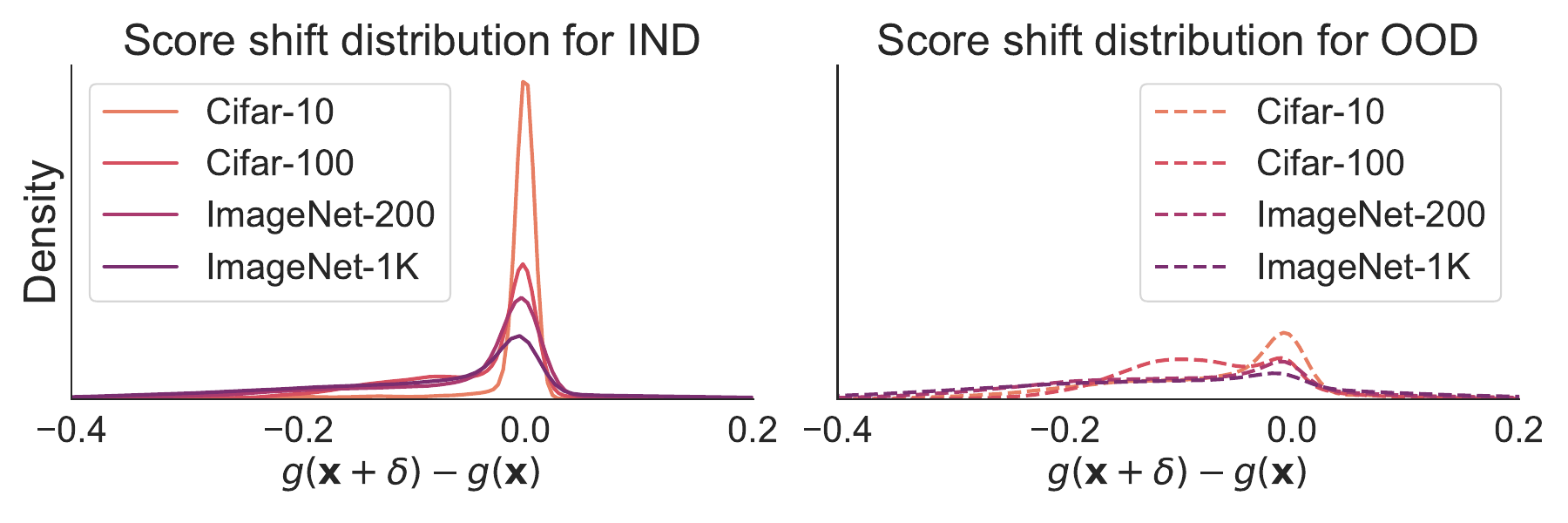}
   \vspace{-0.7cm}
  \caption{Applying the same perturbation $\epsilon=0.001$ leads to different MSP score shifts for different problem scales. The CIFAR-10 model has the best perturbation robustness for IND inputs, other models suffer more IND shift as the class number increases.  }
  
  \label{figure: MSPshift-modelscale}
  \vspace{-0.3cm}
\end{figure}

\section{Conclusion}
In this study, we propose a new OOD detection technique of Perturbation Rectified OOD (\ours) detection. The proposed method stems from an observation that OOD detection scores for OOD inputs are more vulnerable to being attenuated by perturbation. We provide analysis and empirical validation to support the observation. A comprehensive comparison with state-of-the-art baselines demonstrates the effectiveness of \ours, especially its leading performance in distinguishing challenging near-OOD inputs. Furthermore, the increased robustness of the perturbation from adversarial training greatly enhances the performance of OOD detection of \ours. We view our proposed approach as a bridge between adversarial robustness and OOD detection. By leveraging the strengths of both domains, we aim to move towards the safer deployment of deep learning models.

\vspace{0.1cm}
{\bf\noindent Acknowledgments: } 
This work is supported in part by the National Science Foundation under Award \#2420724, the Office of Naval Research under Grant N00014-24-1-2028, and the Army Research Laboratory under Cooperative Agreement Number W911NF-24-2-0163.\footnote{The views and conclusions contained in this document are those of the authors and should not be interpreted as representing the official policies, either expressed or implied, of the Army Research Laboratory or the U.S. Government. The U.S. Government is authorized to reproduce and distribute reprints for Government purposes notwithstanding any copyright notation herein.} 
We thank the anonymous CVPR reviewer for improving the tightness of the bound. %

\clearpage
{
    \small
    \bibliographystyle{ieeenat_fullname}
    \bibliography{main}
}

\clearpage

\maketitlesupplementary
\setcounter{section}{0}
\renewcommand{\theHsection}{A\arabic{section}}
\renewcommand{\thesection}{A\arabic{section}}
\renewcommand{\thetable}{A\arabic{table}}
\setcounter{table}{0}
\setcounter{figure}{0}
\renewcommand{\thetable}{A\arabic{table}}
\renewcommand\thefigure{A\arabic{figure}}
\renewcommand{\theHtable}{A.Tab.\arabic{table}}%
\renewcommand{\theHfigure}{A.Abb.\arabic{figure}}%
\renewcommand\theequation{A\arabic{equation}}
\renewcommand{\theHequation}{A.Abb.\arabic{equation}}%

{\bf \noindent The appendix is organized as follows:}
\begin{itemize}
    \item In~\cref{sec:supp_ent}, we provide analysis following derivations in~\secref{Section: Adversarially robustness for OOD detection}. Similarly, we intend to provide an IND Entropy score bound from adversarial training.
    \item In~\cref{sec:supp_additional result}, we present additional experimental results to better illustrate the effectiveness of \ours.
    \item In~\cref{sec: supp_implementation}, we introduce the implementation details including hardware and hyperparameters.
\end{itemize}
In addition to the appendix, we have attached the code for reference if more details are needed.

\section{Adversarial robustness of entropy score}
\label{sec:supp_ent}

We aim to show the relationship between adversarial robustness to the lower bound of the perturbed IND entropy score by demonstrating that perturbation has a limited effect on attenuating IND entropy scores. 

The entropy score for OOD detection is defined as the negative Shannon entropy, aligning with the conventional setting where higher scores indicate IND inputs:
\begin{equation}
    g_{\tt ENT}(\mathbf{x}) = -H(f (\mathbf{x}))=  \sum_{i=1}^{C} p_i(\mathbf{x}) \log p_i(\mathbf{x})
\end{equation}
\label{Analysis: entropy score}

The analysis follows the derivation for the MSP score in~\secref{Section: Adversarially robustness for OOD detection}. We begin by rewriting the negative prediction entropy in terms of the MSP score and the probabilities of the remaining classes:

\begin{align}
    & -H(f(\mathbf{x}+\delta)) \\
    & = \sum_{i=1}^{C} p_i(f(\mathbf{x}+\delta)) \log p_i(f(\mathbf{x}+\delta)) \notag\\
    &= p_{\text{max}} \log p_{\text{max}} + \sum_{j=2}^{C} p_j \log p_j \notag\\
    &> p_{\text{max}} \log p_{\text{max}} + (C-1) p_a \log p_a, \notag
\end{align}
where $p_a =({1 - p_{\text{max}}})/({C-1})$ denotes the probability evenly distributed among the remaining classes, leading to the maximum prediction entropy given the dominant class probability $p_{\text{max}}$.
Next, we continue to rewrite the lower bound of the entropy score:
\begin{align}
    \Rightarrow & = p_{\text{max}} \log p_{\text{max}} + (1 - p_{\text{max}}) \log \frac{1 - p_{\text{max}}}{C-1} \\
    & = p_{\text{max}} \log p_{\text{max}} + (1 - p_{\text{max}})\log(1 - p_{\text{max}}) \notag\\
     &\quad + p_{\text{max}} \log(C-1) -  \log(C-1)\notag
\end{align}

Denote $h(p)=p \log p +(1-p)\log(1-p)+p\log(C-1) -log(C-1)$, this function is convex and non-decreasing when $c\in [1/C,1]$. Apply Jensen's inequality, and substitute ~\equref{equation: lowerboundMSP}, we have \footnote{We thank the anonymous reviewer for their helpful suggestion regarding the derivation in ~\equref{equ: entropy score lower bound}.}:
\begin{equation}
\label{equ: entropy score lower bound}
    E[-H(f (\mathbf{x}))]\ge E[h(p_{\text{max}})] \ge h(E[p_{\text{max}}]) \ge h(\exp((-\mathcal{E})))
\end{equation}

\section{Additional results}
\label{sec:supp_additional result}
{\bf\noindent Perturbation robustness analysis.} We extend the analysis of robustness differences using the metric of score shift. Similar to \figref{figure: dz-steplength}, we evaluate score shifts under one-step perturbations of varying magnitudes. In \figref{figure: dz-cifar10-alloods}, the MSP score shifts are shown for a CIFAR-10 model without adversarial training. These results illustrate that OOD scores are generally more susceptible to perturbations compared to IND scores even without adversarial training. Additionally, we analyze score shifts on ImageNet models in \figref{figure: dz-imagenet-default} and \figref{figure: dz-imagenet-pixmix}. While a significant proportion of IND scores remain robust, forming a peak distribution near zero, a notable portion of IND scores still experience significant decreases under perturbation.

{\bf\noindent Score distribution shift.} To provide further intuition on how PRO reshapes the original score distribution, \figref{fig: distribution change} compares the PRO-enhanced scores with original MSP and ENT scores. As demonstrated in the plots, PRO effectively reduces the score values for OOD inputs, resulting in a distribution shift toward lower values. However, we can observe that the shifts also happened within IND scores. These shifts are particularly notable for the ImageNet model, especially MSP scores, limiting the enhancement from PRO. 

{\bf\noindent OOD detection performance on ImageNet.}
Detailed OOD detection performance metrics for ImageNet are provided in \tabref{table: OOD performance ImageNet}. We present a default model and three models trained with data augmentation procedures PixMix~\cite{hendrycks2022pixmix}, AugMix~\cite{hendrycks*2020augmix}, and RegMixup~\cite{pinto2022using}. %
We focus on the comparison with softmax scores and other gradient-based methods. The gradient-based method GradNorm~\cite{huang2021importance} shows significant performance degradation for models trained with PixMix and AugMix, indicating that gradients with respect to weights are highly sensitive to data augmentations. ODIN exhibits reduced far-OOD performance across all models compared to the MSP baseline.

In contrast, our proposed method, \ours, provides consistent improvements over basic scores such as MSP and Entropy, establishing \ours as the most competitive post-hoc method for near-OOD detection among the compared baselines. However, the effect of \ours on Temperature-scaled MSP and GEN is inconsistent across models. We attribute this variability to the additional hyperparameters in these methods, which increase dependence on the evaluation set's comprehensiveness.

{\bf\noindent Additional metrics on CIFAR-10.} 
\tabref{table: supp-cifar10robustmodels} provides the OOD detection performance tested on the other three robust models as an extension to \tabref{Table: Ciafr-10 OOD detection}. Both LRR-CARD-Deck and Binary-CARD-Deck have adopted an ensemble of models, making most post-hoc methods perform similarly to MSP baseline. We average the activations between models in an ensemble to implement Scale, Ash, and React. The binary model has an unconventional linear layer thus we have not implemented activation-modification methods on it. \ours has improved most averaged metrics of four softmax scores on Augmix models, achieving leading performance among baselines.

{\bf\noindent Metrics on different CIFAR-100 models.}
We present OOD detection metrics on five different CIFAR-100 models in \tabref{Table: supp-cifar100-robustmodels}. For this analysis, we focus on original softmax scores and ODIN as baselines to emphasize the enhancements achieved by \ours across different models. For a detailed comparison with other representative state-of-the-art methods, please refer to \tabref{table: cifar-100}. \ours provides consistent improvements across different models, particularly for temperature-scaled confidence, entropy, and GEN. As shown in the averaged metrics, \ours-MSP-T, \ours-MSP-ENT, and \ours-GEN demonstrate leading performance across most models.

\begin{figure}[t]
  \includegraphics[width=0.49\textwidth]{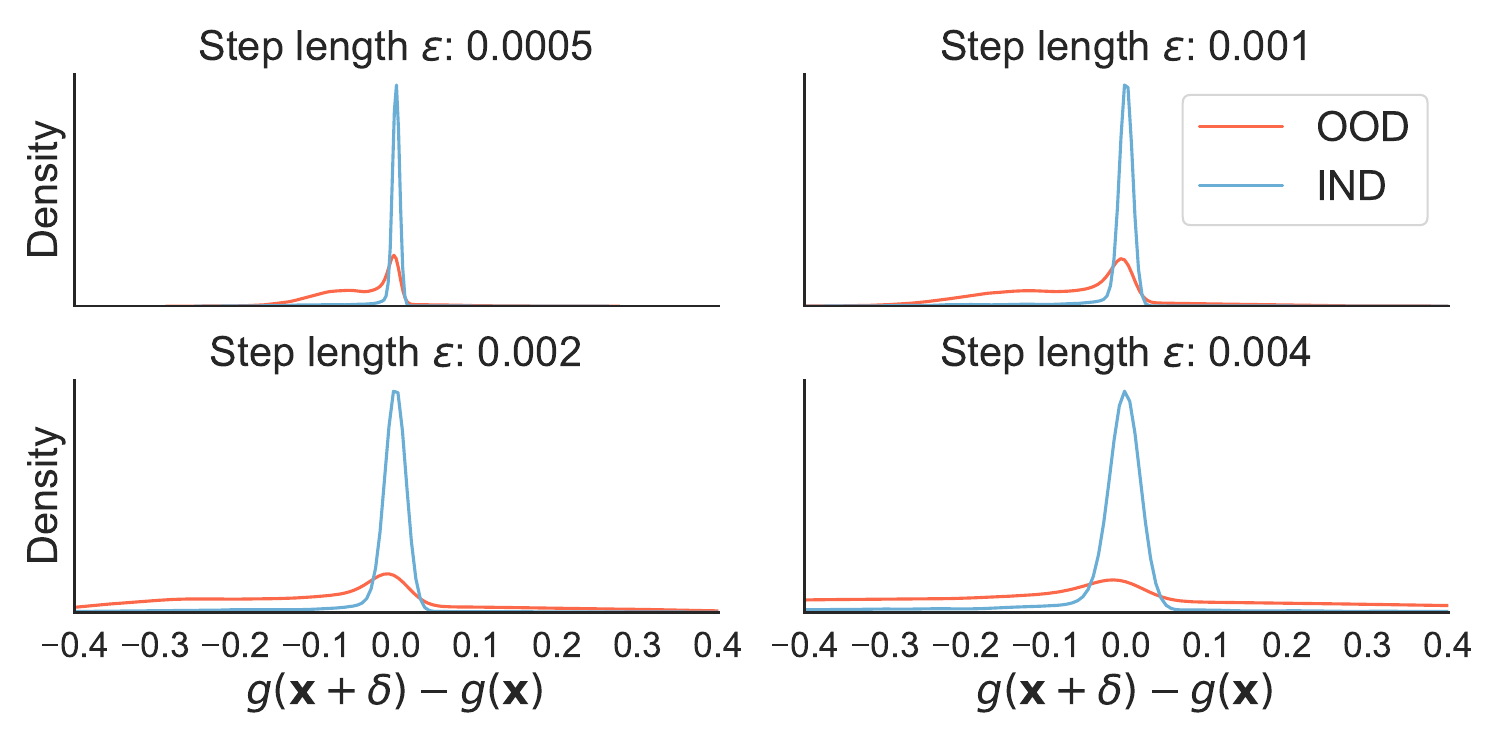}
  \centering
    \vspace{-0.3cm}
  \caption{Distribution plots of MSP score shift introduced by a bounded perturbation. It is tested on a robust CIFAR-10 model without adversarial training.}
\vspace{-0.3cm}
  \label{figure: dz-cifar10-alloods}
\end{figure}
\begin{figure}[t]
  \includegraphics[width=0.49\textwidth]{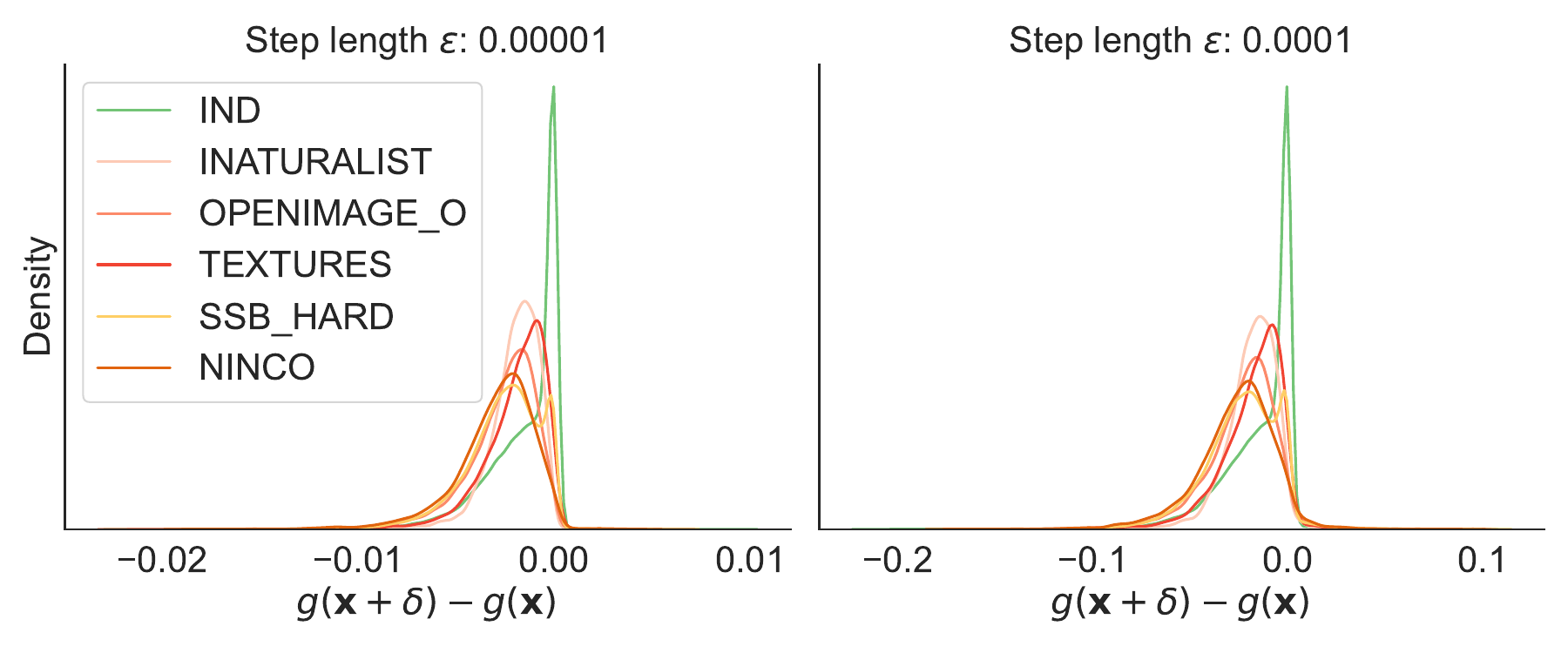}
  \centering
    \vspace{-0.7cm}
  \caption{Distribution plots of MSP score shift introduced by one-step perturbation on a default ImageNet model without adversarial training.}
\vspace{-0.3cm}
  \label{figure: dz-imagenet-default}
\end{figure}

\begin{figure}[t]
  \includegraphics[width=0.49\textwidth]{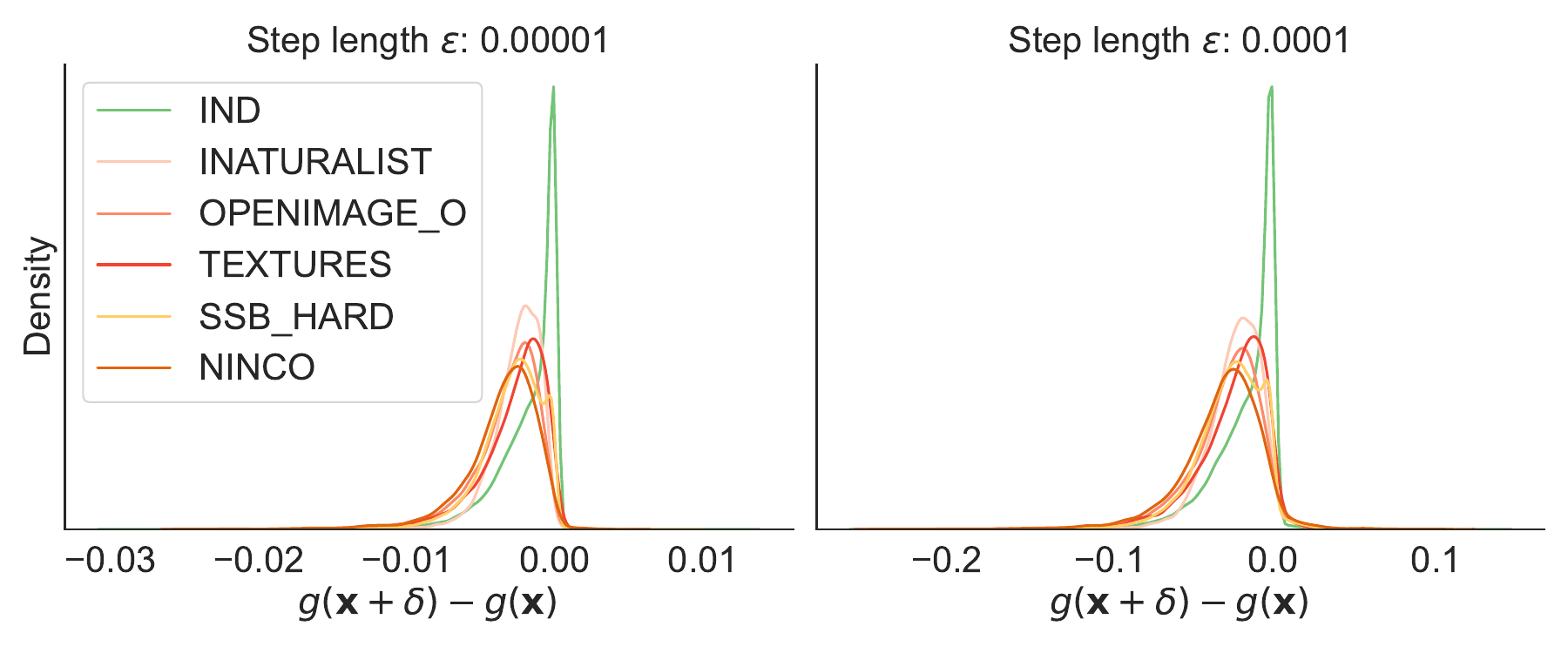}
  \centering
    \vspace{-0.7cm}
  \caption{Distribution plots of MSP score shift introduced by one-step perturbation on an ImageNet model trained with PixMix~\cite{hendrycks2022pixmix} data augmentation}
\vspace{-0.3cm}
  \label{figure: dz-imagenet-pixmix}
\end{figure}

\begin{figure*}[!t]
    \centering

    \begin{subfigure}[b]{0.4\textwidth}
        \centering
        \includegraphics[width=\textwidth]{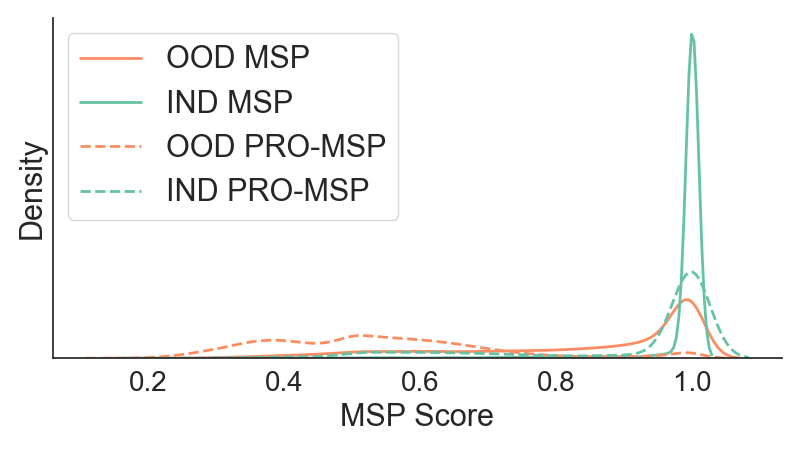} 
        \caption{CIFAR-10: LRR \cite{diffenderfer2021winning}}
    \end{subfigure}
    \hspace{0.2cm}
    \begin{subfigure}[b]{0.4\textwidth}
        \centering
        \includegraphics[width=\textwidth]{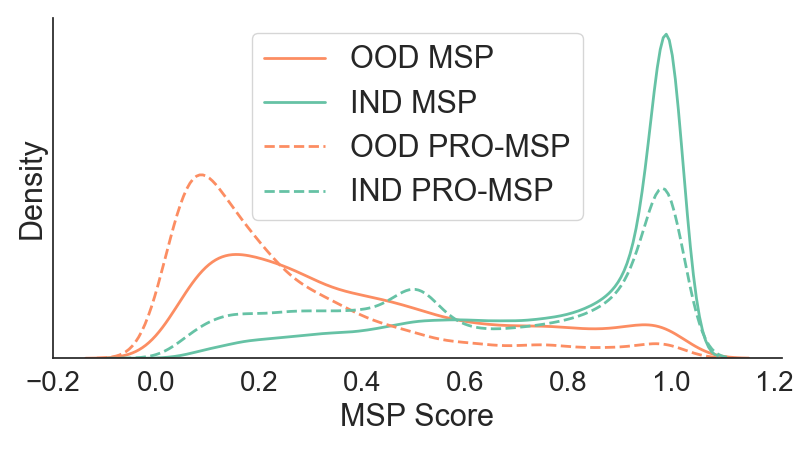}   
        \caption{ImageNet: AugMix\cite{hendrycks*2020augmix}}
    \end{subfigure}

    \label{fig: }
    \begin{subfigure}[b]{0.4\textwidth}
        \centering
        \includegraphics[width=\textwidth]{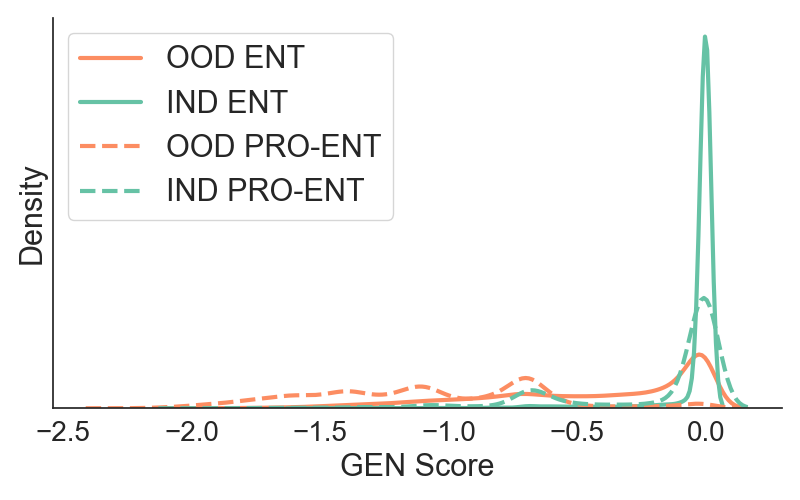} 
        \caption{CIFAR-10: LRR \cite{diffenderfer2021winning}}
    \end{subfigure}
    \hspace{0.2cm}
    \begin{subfigure}[b]{0.4\textwidth}
        \centering
        \includegraphics[width=\textwidth]{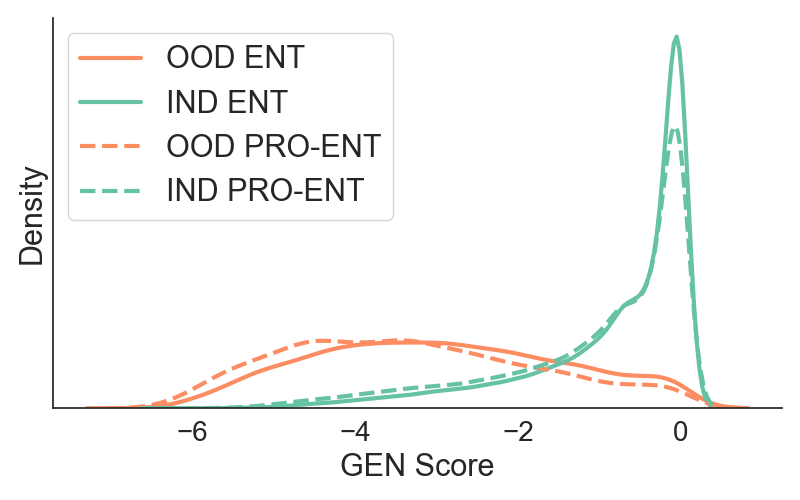}   
        \caption{ImageNet: AugMix\cite{hendrycks*2020augmix}}
    \end{subfigure}
    \caption{PRO method reshapes scores distribution. We select MSP and ENT from a robust CIFAR-10 model and a robust ImageNet model. In above plots, OOD for CIFAR-10 is SVHN \cite{netzer2011reading} and OOD for ImageNet is Texture\cite{cimpoi2014describing}.} 
    \label{fig: distribution change}
\end{figure*}

\begin{table*}[ht]
\centering

\setlength{\tabcolsep}{4pt}  
\resizebox{\textwidth}{!}{%
\begin{tabular}{l cc cc cc cc} \toprule
 & \multicolumn{2}{c}{\textbf{Default Model}} & \multicolumn{2}{c}{\textbf{PixMix}} & \multicolumn{2}{c}{\textbf{AugMix}} & \multicolumn{2}{c}{\textbf{RegMixup}} \\
Method & Near-OOD & Far-OOD & Near-OOD & Far-OOD & Near-OOD & Far-OOD & Near-OOD & Far-OOD \\
\midrule
MSP\cite{hendrycks2016baseline} & 65.68/76.02 & 51.45/85.23 & 65.89/76.86 & 51.11/85.63 & 64.45/77.49 & 46.94/86.67 & 65.33/77.04 & 48.91/86.31 \\
TempScaling\cite{guo2017calibration} & \textbf{64.5}/77.14 & 46.64/87.56 & 64.85/78.02 & 46.82/87.59 & 62.61/78.57 & 42.07/88.75 & 64.26/77.87 & 44.6/87.95 \\
Entropy\cite{hendrycks2016baseline} & 64.96/\underline{77.38} & 47.86/88.01 & 64.69/78.38 & 46.16/88.41 & 63.16/78.78 & 41.81/89.41 & 63.69/\underline{78.24} & 41.9/88.95 \\
GEN\cite{liu2023gen} & 65.32/76.85 & \textbf{35.61}/\underline{89.76} & 66.77/77.78 & \underline{38.13}/\textbf{89.54} & 64.0/78.72 & \underline{32.98}/\underline{90.99} & \underline{63.16}/77.65 & \textbf{34.78}/\textbf{89.65} \\
ODIN\cite{liang2017enhancing} & 72.5/74.75 & 43.96/89.47 & 75.32/74.32 & 61.36/84.45 & 67.71/77.69 & 36.52/\textbf{91.1} & 74.5/75.18 & 49.47/88.79 \\
GradNorm\cite{huang2021importance} & 78.89/72.96 & 47.92/\textbf{90.25} & 85.37/63.42 & 79.68/72.27 & 76.3/72.14 & 60.35/85.01 & 81.96/69.22 & 58.99/85.75 \\
MLS\cite{hendrycks2019scaling} & 67.82/76.46 & 38.22/89.57 & 67.57/78.28 & 41.36/89.21 & 63.36/\textbf{79.14} & 33.47/90.87 & 67.99/77.43 & 38.93/89.25 \\
EBO\cite{liu2020energy} & 68.56/75.89 & 38.39/89.47 & 68.75/77.75 & 41.04/89.3 & 64.17/78.76 & 33.45/90.95 & 69.06/76.48 & 39.97/88.87 \\
RankFeat\cite{song2022rankfeat} & 91.83/50.99 & 87.17/53.93 & 95.36/42.27 & 90.32/42.62 & 93.09/51.18 & 81.14/60.44 & 96.92/41.4 & 94.68/38.39 \\
\rowcolor{blue!15} PRO-MSP & 65.0/76.9 & 52.87/85.54 & \underline{63.36}/77.66 & 47.2/87.15 & 63.49/78.21 & 47.77/87.01 & 64.59/77.58 & 50.87/86.2 \\
\rowcolor{blue!15} PRO-MSP-T & 67.5/76.54 & 37.96/89.61 & 65.21/\underline{78.77} & 40.19/88.92 & 63.33/\underline{79.14} & 33.48/90.86 & 67.59/77.5 & 38.61/\underline{89.29} \\
\rowcolor{blue!15} PRO-ENT & \underline{64.55}/\textbf{77.66} & 46.57/87.85 & \textbf{61.71}/\textbf{78.8} & 41.78/88.49 & \underline{62.41}/79.01 & 39.85/89.24 & 63.52/\textbf{78.26} & 41.73/88.9 \\
\rowcolor{blue!15} PRO-GEN & 65.13/76.62 & \underline{37.21}/89.32 & 64.05/78.2 & \textbf{37.57}/\underline{89.37} & \textbf{62.08}/78.56 & \textbf{32.35}/90.65 & \textbf{62.96}/77.48 & \underline{35.82}/88.99 \\
\bottomrule
\end{tabular}
}
\caption{OOD detection performance on ImageNet. Besides general table format \textbf{best metric}, \underline{second best metric}, and \colorbox{blue!15}{our methods}.
}
\label{table: OOD performance ImageNet}
\end{table*}

\section{Implementation details \& hyperparameter}
\label{sec: supp_implementation}

All experiments presented in this work are conducted on a workstation with four NVIDIA RTX 2080 Ti GPUs and an Intel CPU running at 2.90 GHz. The results can be reproduced by following the experimental platform established by the OpenOOD benchmark~\cite{yang2022openood} \footnote{https://github.com/Jingkang50/OpenOOD}. 

In addition to the overview of \ours provided in Algorithm \ref{alg:g}, we highlight a few additional implementation details. \ours has two hyperparameters which are determined by the evaluation set of test benchmarks. We consider step length $\epsilon$ to be within the range of 0.00005 to 0.01, with perturbations applied to normalized image tensors. As for update steps $K$, we limit it to a maximum of 7 to manage computational overhead. Additional hyperparameters introduced by temperature scaling and GEN have reduced search space for higher efficiency. \tabref{Table: hyperparameters} provides the considered hyperparameter settings for different datasets. It is important to note that the optimal hyperparameters may vary across different pre-trained models.

{\bf\noindent Hyperparameter sensitivity analysis.} 
Please see \figref{figure: hyperparametersensitivity} as an ablation study on hyperparameters. The key takeaway here is to use a small perturbation step $\epsilon$ which stably improves performance as step number $K$ increases.

\begin{table*}[ht]
\centering
\resizebox{\textwidth}{!}{%
\begin{tabular}{p{0.3cm}lp{1.6cm}p{1.6cm}p{1.6cm}p{1.6cm}p{1.6cm}p{1.6cm}p{1.6cm}}
\toprule
\multicolumn{2}{c}{\textbf{IND: CIFAR-10}} & \multicolumn{7}{c}{\textbf{OOD detection performance: FPR@95 $ \downarrow $ / AUROC $ \uparrow $}}  \\
\midrule
{} & Method & Cifar100 & TIN & MNIST & SVHN & Texture & Places365 & Average \\
\midrule
\multirow{14}{*}{\rotatebox{90}{LRR-CARD-Deck}} & MSP\cite{hendrycks2016baseline}  & \textbf{27.76}/91.72 & \textbf{22.86}/92.99 & \underline{19.84}/93.79 & 17.90/94.27 & \underline{16.37}/95.28 & \textbf{23.72}/92.80 & \textbf{21.41}/93.47 \\
& TempScaling\cite{guo2017calibration} & \textbf{27.76}/91.72 & \textbf{22.86}/92.99 & \underline{19.84}/93.80 & 17.90/94.27 & \underline{16.37}/95.29 & \textbf{23.72}/92.80 & \textbf{21.41}/93.48 \\
& Entropy\cite{hendrycks2016baseline} & \textbf{27.76}/91.85 & \textbf{22.86}/93.15 & \underline{19.84}/94.02 & 17.90/94.41 & \underline{16.37}/95.49 & \textbf{23.72}/92.97 & \textbf{21.41}/93.65 \\
& GEN\cite{liu2023gen} & \textbf{27.76}/91.87 & 22.89/93.17 & \underline{19.84}/\underline{94.05} & 17.91/94.42 & \underline{16.37}/\textbf{95.51} & \textbf{23.72}/93.00 & 21.42/93.67 \\
& ODIN\cite{liang2017enhancing} & 32.51/91.03 & 27.02/92.28 & \textbf{13.40}/\textbf{96.29} & 19.68/94.11 & \textbf{15.96}/95.49 & 28.24/91.94 & 22.80/93.52 \\
& MLS\cite{hendrycks2019scaling} & \textbf{27.76}/91.73 & \textbf{22.86}/93.00 & \underline{19.84}/93.81 & 17.90/94.28 & \underline{16.37}/95.30 & \textbf{23.72}/92.81 & \textbf{21.41}/93.49 \\
& EBO\cite{liu2020energy} & \textbf{27.76}/91.87 & 22.87/93.18 & \underline{19.84}/\underline{94.05} & 17.91/94.43 & \underline{16.37}/\textbf{95.51} & \textbf{23.72}/93.00 & \textbf{21.41}/93.67 \\
& ASH\cite{djurisic2023extremely} & 70.56/79.54 & 68.27/81.34 & 50.40/88.65 & 82.06/65.88 & 61.62/85.02 & 57.06/85.46 & 64.99/80.98 \\
& ReAct\cite{sun2021react} & 74.18/74.78 & 71.72/78.00 & 59.22/82.77 & 82.07/59.04 & 73.30/75.50 & 58.23/85.20 & 69.79/75.88 \\
& Scale\cite{xu2024scaling} & 64.16/83.67 & 60.18/85.64 & 49.46/88.89 & 74.01/76.93 & 56.48/87.22 & 50.30/88.56 & 59.10/85.15 \\
& \cellcolor{blue!15}PRO-MSP & \cellcolor{blue!15}29.01/92.09 & \cellcolor{blue!15}23.41/93.61 & \cellcolor{blue!15}28.76/92.27 & \cellcolor{blue!15}\textbf{12.62}/95.71 & \cellcolor{blue!15}20.29/94.87 & \cellcolor{blue!15}24.43/93.53 & \cellcolor{blue!15}23.09/93.68 \\
& \cellcolor{blue!15}PRO-MSP-T & \cellcolor{blue!15}29.64/91.83 & \cellcolor{blue!15}24.30/93.48 & \cellcolor{blue!15}28.86/92.14 & \cellcolor{blue!15}14.64/95.48 & \cellcolor{blue!15}22.72/94.67 & \cellcolor{blue!15}25.84/93.43 & \cellcolor{blue!15}24.33/93.50 \\
& \cellcolor{blue!15}PRO-ENT & \cellcolor{blue!15}28.82/\textbf{92.45} & \cellcolor{blue!15}23.47/\underline{94.08} & \cellcolor{blue!15}28.46/93.13 & \cellcolor{blue!15}14.43/\underline{95.76} & \cellcolor{blue!15}19.93/95.43 & \cellcolor{blue!15}24.24/\underline{94.07} & \cellcolor{blue!15}23.23/\textbf{94.15} \\
& \cellcolor{blue!15}PRO-GEN & \cellcolor{blue!15}29.57/\underline{92.37} & \cellcolor{blue!15}24.08/\textbf{94.09} & \cellcolor{blue!15}28.62/93.16 & \cellcolor{blue!15}\underline{14.06}/\textbf{95.79} & \cellcolor{blue!15}21.62/95.28 & \cellcolor{blue!15}25.50/\textbf{94.11} & \cellcolor{blue!15}23.91/\underline{94.13} \\
\midrule
\multirow{11}{*}{\rotatebox{90}{Binary-CARD-Deck}} & MSP\cite{hendrycks2016baseline}  & \textbf{31.58}/90.25 & \underline{27.87}/91.24 & 17.26/95.32 & 23.33/92.10 & \textbf{21.39}/93.45 & \textbf{29.47}/91.04 & 25.15/92.23 \\
& TempScaling\cite{guo2017calibration} & \textbf{31.58}/90.26 & \underline{27.87}/91.24 & 17.26/95.33 & 23.33/92.10 & \textbf{21.39}/93.46 & \textbf{29.47}/91.05 & 25.15/92.24 \\
& Entropy\cite{hendrycks2016baseline} & \textbf{31.58}/90.50 & \underline{27.87}/91.52 & 17.18/95.89 & 23.33/92.25 & 21.42/93.82 & \textbf{29.47}/91.36 & \underline{25.14}/92.56 \\
& GEN\cite{liu2023gen} & \textbf{31.58}/\underline{90.53} & \underline{27.87}/91.56 & \underline{17.17}/\underline{95.97} & 23.33/92.28 & 21.46/\underline{93.86} & 29.48/91.40 & 25.15/92.60 \\
& ODIN\cite{liang2017enhancing} & 32.37/90.17 & 29.59/90.84 & \textbf{6.86}/\textbf{98.53} & 25.42/91.50 & 23.53/93.54 & 30.83/90.70 & \textbf{24.77}/92.55 \\
& MLS\cite{hendrycks2019scaling} & \textbf{31.58}/90.27 & \underline{27.87}/91.26 & 17.26/95.36 & 23.33/92.11 & \textbf{21.39}/93.48 & \textbf{29.47}/91.07 & 25.15/92.26 \\
& EBO\cite{liu2020energy} & \textbf{31.58}/\textbf{90.54} & \textbf{27.86}/91.57 & \underline{17.17}/\underline{95.97} & 23.31/92.28 & 21.43/\textbf{93.87} & \textbf{29.47}/91.41 & \underline{25.14}/\underline{92.61} \\
& \cellcolor{blue!15}PRO-MSP & \cellcolor{blue!15}35.08/89.67 & \cellcolor{blue!15}30.52/91.29 & \cellcolor{blue!15}31.03/92.45 & \cellcolor{blue!15}18.56/93.35 & \cellcolor{blue!15}27.98/92.50 & \cellcolor{blue!15}31.58/91.38 & \cellcolor{blue!15}29.12/91.77 \\
& \cellcolor{blue!15}PRO-MSP-T & \cellcolor{blue!15}34.01/89.98 & \cellcolor{blue!15}29.77/91.47 & \cellcolor{blue!15}27.72/92.93 & \cellcolor{blue!15}\textbf{17.37}/\underline{93.78} & \cellcolor{blue!15}24.96/92.97 & \cellcolor{blue!15}30.63/91.50 & \cellcolor{blue!15}27.41/92.11 \\
& \cellcolor{blue!15}PRO-ENT & \cellcolor{blue!15}34.84/90.17 & \cellcolor{blue!15}30.20/\underline{91.94} & \cellcolor{blue!15}30.51/93.55 & \cellcolor{blue!15}20.17/93.30 & \cellcolor{blue!15}27.34/93.21 & \cellcolor{blue!15}31.42/\underline{92.07} & \cellcolor{blue!15}29.08/92.37 \\
& \cellcolor{blue!15}PRO-GEN & \cellcolor{blue!15}33.88/90.43 & \cellcolor{blue!15}29.57/\textbf{92.03} & \cellcolor{blue!15}27.14/93.98 & \cellcolor{blue!15}\underline{17.40}/\textbf{93.91} & \cellcolor{blue!15}24.68/93.64 & \cellcolor{blue!15}30.61/\textbf{92.11} & \cellcolor{blue!15}27.21/\textbf{92.68} \\
\midrule
\multirow{13}{*}{\rotatebox{90}{AugMix-ResNeXt}} & MSP\cite{hendrycks2016baseline}  & 29.66/91.03 & 26.22/92.03 & 13.66/96.09 & 27.87/90.84 & \underline{27.79}/91.46 & 25.93/92.12 & 25.19/92.26 \\
& TempScaling\cite{guo2017calibration} & \textbf{29.11}/91.42 & 25.49/92.48 & 12.66/96.75 & 27.91/90.96 & \textbf{27.56}/91.83 & 25.31/92.61 & 24.67/92.67 \\
& Entropy\cite{hendrycks2016baseline} & 29.38/91.51 & 25.90/92.59 & 13.09/97.02 & 27.83/91.04 & 27.80/91.94 & 25.63/92.72 & 24.94/92.80 \\
& GEN\cite{liu2023gen} & 29.43/\textbf{92.13} & \underline{23.51}/\underline{93.63} & 6.43/98.60 & 33.93/89.04 & 29.30/91.90 & 22.76/94.03 & \underline{24.23}/\underline{93.22} \\
& ODIN\cite{liang2017enhancing} & 42.48/89.40 & 39.19/90.18 & \textbf{0.97}/\textbf{99.74} & 77.19/73.85 & 51.96/87.73 & 32.21/91.64 & 40.67/88.76 \\
& MLS\cite{hendrycks2019scaling} & 29.92/92.08 & 23.91/93.59 & 6.56/98.51 & 36.03/88.82 & 29.72/91.80 & 22.73/94.00 & 24.81/93.13 \\
& EBO\cite{liu2020energy} & 29.90/92.04 & 23.97/93.60 & 6.04/98.67 & 36.12/88.52 & 29.98/91.70 & \underline{22.71}/\underline{94.04} & 24.79/93.09 \\
& ASH\cite{djurisic2023extremely} & 35.47/90.95 & 29.67/92.34 & 4.18/99.13 & 54.46/84.75 & 30.01/\underline{92.59} & 22.84/93.80 & 29.44/92.26 \\
& Scale\cite{xu2024scaling} & 34.53/91.59 & 28.28/93.07 & \underline{3.46}/\underline{99.22} & 56.68/85.31 & 28.70/\textbf{93.13} & 23.24/94.02 & 29.15/92.72 \\
& \cellcolor{blue!15}PRO-MSP & \cellcolor{blue!15}30.23/90.70 & \cellcolor{blue!15}26.89/91.95 & \cellcolor{blue!15}19.27/94.53 & \cellcolor{blue!15}\textbf{17.23}/\textbf{93.31} & \cellcolor{blue!15}30.13/90.95 & \cellcolor{blue!15}27.02/92.06 & \cellcolor{blue!15}25.13/92.25 \\
& \cellcolor{blue!15}PRO-MSP-T & \cellcolor{blue!15}29.90/92.08 & \cellcolor{blue!15}23.73/93.61 & \cellcolor{blue!15}6.67/98.47 & \cellcolor{blue!15}34.67/89.32 & \cellcolor{blue!15}29.94/91.81 & \cellcolor{blue!15}22.73/94.01 & \cellcolor{blue!15}24.61/\underline{93.22} \\
& \cellcolor{blue!15}PRO-ENT & \cellcolor{blue!15}30.96/91.37 & \cellcolor{blue!15}27.22/92.99 & \cellcolor{blue!15}19.50/96.21 & \cellcolor{blue!15}\underline{24.76}/\underline{92.10} & \cellcolor{blue!15}32.24/91.37 & \cellcolor{blue!15}27.52/93.18 & \cellcolor{blue!15}27.03/92.87 \\
& \cellcolor{blue!15}PRO-GEN & \cellcolor{blue!15}\underline{29.36}/\underline{92.12} & \cellcolor{blue!15}\textbf{23.46}/\textbf{93.66} & \cellcolor{blue!15}6.61/98.55 & \cellcolor{blue!15}31.97/89.87 & \cellcolor{blue!15}29.53/91.89 & \cellcolor{blue!15}\textbf{22.60}/\textbf{94.05} & \cellcolor{blue!15}\textbf{23.92}/\textbf{93.36} \\
\bottomrule
\end{tabular}%
}
\caption{OOD detection performance on three CIFAR-10 robust models.}
\label{table: supp-cifar10robustmodels}
\end{table*}

\begin{table*}[ht]
\centering
\resizebox{\textwidth}{!}{%
\begin{tabular}{p{0.3cm}lp{1.6cm}p{1.6cm}p{1.6cm}p{1.6cm}p{1.6cm}p{1.6cm}p{1.6cm}}
\toprule
\multicolumn{2}{c}{\textbf{IND: CIFAR-100}} & \multicolumn{7}{c}{\textbf{OOD detection performance: FPR@95 $ \downarrow $ / AUROC $ \uparrow $}}  \\
\midrule
{} & Method & Cifar10 & TIN & MNIST & SVHN & Texture & Places365 & Average \\
\midrule
\multirow{9}{*}{\rotatebox{90}{Default Model}} & MSP\cite{hendrycks2016baseline}  & 58.91/78.47 & 50.70/82.07 & 57.23/76.08 & 59.07/78.42 & 61.88/77.32 & 56.62/79.22 & 57.40/78.60 \\
& TempScaling\cite{guo2017calibration} & \textbf{58.72}/79.02 & 50.26/82.79 & 56.05/77.27 & 57.71/79.79 & \underline{61.56}/78.11 & 56.46/79.80 & 56.79/79.46 \\
& Entropy\cite{hendrycks2016baseline} & \underline{58.83}/79.21 & 50.33/83.08 & 56.73/77.46 & 58.47/80.11 & 61.68/78.32 & 56.43/79.99 & 57.08/79.70 \\
& GEN\cite{liu2023gen} & 58.87/\textbf{79.38} & \underline{49.98}/83.25 & \underline{53.92}/\underline{78.29} & 55.45/81.41 & \textbf{61.23}/\underline{78.74} & \underline{56.25}/\textbf{80.28} & 55.95/\underline{80.23} \\
& ODIN\cite{liang2017enhancing} & 60.64/78.18 & 55.19/81.63 & \textbf{45.94}/\textbf{83.79} & 67.41/74.54 & 62.37/\textbf{79.33} & 59.71/79.45 & 58.54/79.49 \\
& \cellcolor{blue!15}PRO-MSP & \cellcolor{blue!15}60.84/78.75 & \cellcolor{blue!15}51.36/82.82 & \cellcolor{blue!15}62.38/73.31 & \cellcolor{blue!15}48.30/84.35 & \cellcolor{blue!15}66.45/75.91 & \cellcolor{blue!15}57.00/79.47 & \cellcolor{blue!15}57.72/79.10 \\
& \cellcolor{blue!15}PRO-MSP-T & \cellcolor{blue!15}60.18/79.05 & \cellcolor{blue!15}51.13/83.03 & \cellcolor{blue!15}56.13/76.32 & \cellcolor{blue!15}\textbf{44.29}/85.48 & \cellcolor{blue!15}64.43/77.46 & \cellcolor{blue!15}57.24/79.59 & \cellcolor{blue!15}\textbf{55.57}/80.15 \\
& \cellcolor{blue!15}PRO-ENT & \cellcolor{blue!15}60.17/79.09 & \cellcolor{blue!15}50.21/\underline{83.34} & \cellcolor{blue!15}60.69/74.72 & \cellcolor{blue!15}\underline{46.62}/\textbf{86.06} & \cellcolor{blue!15}64.77/77.21 & \cellcolor{blue!15}56.63/79.79 & \cellcolor{blue!15}56.51/80.04 \\
& \cellcolor{blue!15}PRO-GEN & \cellcolor{blue!15}59.83/\underline{79.24} & \cellcolor{blue!15}\textbf{49.62}/\textbf{83.47} & \cellcolor{blue!15}58.07/75.80 & \cellcolor{blue!15}46.81/\underline{85.51} & \cellcolor{blue!15}63.45/77.85 & \cellcolor{blue!15}\textbf{56.18}/\underline{80.07} & \cellcolor{blue!15}\underline{55.66}/\textbf{80.32} \\
\midrule
\multirow{9}{*}{\rotatebox{90}{Robust Model LRR}} & MSP\cite{hendrycks2016baseline}  & \underline{57.19}/78.88 & 50.36/81.49 & 57.46/74.67 & 52.73/78.87 & 62.81/74.53 & 56.52/78.17 & 56.18/77.77 \\
& TempScaling\cite{guo2017calibration} & 57.48/79.91 & 49.02/83.07 & 55.00/77.96 & 52.17/79.79 & 62.31/75.62 & 56.14/79.15 & 55.35/79.25 \\
& Entropy\cite{hendrycks2016baseline} & \textbf{57.03}/79.85 & 49.97/82.97 & 56.83/77.01 & 52.50/79.73 & 62.83/75.31 & 56.43/79.07 & 55.93/78.99 \\
& GEN\cite{liu2023gen} & 58.52/\textbf{80.68} & \textbf{46.41}/\textbf{84.43} & 49.08/80.70 & 47.88/81.82 & \textbf{60.02}/\textbf{77.30} & \underline{54.01}/\underline{80.56} & \underline{52.65}/\underline{80.92} \\
& ODIN\cite{liang2017enhancing} & 68.01/76.36 & 56.21/80.59 & \textbf{20.62}/\textbf{94.96} & 75.73/66.27 & 70.17/73.40 & 65.94/74.85 & 59.45/77.74 \\
& \cellcolor{blue!15}PRO-MSP & \cellcolor{blue!15}59.16/79.14 & \cellcolor{blue!15}51.32/82.73 & \cellcolor{blue!15}69.49/70.39 & \cellcolor{blue!15}51.13/81.35 & \cellcolor{blue!15}69.44/74.01 & \cellcolor{blue!15}55.88/78.97 & \cellcolor{blue!15}59.40/77.77 \\
& \cellcolor{blue!15}PRO-MSP-T & \cellcolor{blue!15}61.94/79.94 & \cellcolor{blue!15}49.18/84.01 & \cellcolor{blue!15}\underline{47.60}/\underline{83.33} & \cellcolor{blue!15}\underline{39.06}/84.15 & \cellcolor{blue!15}64.18/76.02 & \cellcolor{blue!15}57.11/79.09 & \cellcolor{blue!15}53.18/\textbf{81.09} \\
& \cellcolor{blue!15}PRO-ENT & \cellcolor{blue!15}58.64/80.23 & \cellcolor{blue!15}49.98/84.12 & \cellcolor{blue!15}66.20/74.01 & \cellcolor{blue!15}42.50/\underline{84.58} & \cellcolor{blue!15}67.42/75.45 & \cellcolor{blue!15}55.23/80.10 & \cellcolor{blue!15}56.66/79.75 \\
& \cellcolor{blue!15}PRO-GEN & \cellcolor{blue!15}59.57/\underline{80.47} & \cellcolor{blue!15}\underline{46.63}/\underline{84.41} & \cellcolor{blue!15}55.73/76.58 & \cellcolor{blue!15}\textbf{37.30}/\textbf{85.16} & \cellcolor{blue!15}\underline{61.18}/\underline{76.86} & \cellcolor{blue!15}\textbf{52.58}/\textbf{80.84} & \cellcolor{blue!15}\textbf{52.16}/80.72 \\
\midrule
\multirow{9}{*}{\rotatebox{90}{Binary}} & MSP\cite{hendrycks2016baseline}  & \textbf{62.81}/77.05 & 53.92/80.85 & 71.29/67.01 & 51.81/80.96 & 69.90/74.64 & 59.13/78.02 & 61.48/76.42 \\
& TempScaling\cite{guo2017calibration} & 63.20/77.75 & 53.03/81.90 & 69.83/69.10 & 49.90/82.22 & 67.57/76.06 & 57.87/79.04 & 60.23/77.68 \\
& Entropy\cite{hendrycks2016baseline} & \underline{63.08}/78.41 & 52.99/82.81 & 70.48/70.51 & 49.71/83.29 & 68.00/77.04 & 58.02/79.88 & 60.38/78.66 \\
& GEN\cite{liu2023gen} & 64.22/\underline{78.63} & 50.74/\underline{83.40} & 66.94/74.32 & 45.08/84.00 & \underline{64.97}/78.93 & 55.94/80.72 & 57.98/80.00 \\
& ODIN\cite{liang2017enhancing} & 73.06/73.28 & 67.00/77.42 & \textbf{32.00}/\textbf{91.08} & 82.71/64.36 & 67.63/78.46 & 66.87/76.59 & 64.88/76.86 \\
& \cellcolor{blue!15}PRO-MSP & \cellcolor{blue!15}63.39/77.91 & \cellcolor{blue!15}53.43/82.02 & \cellcolor{blue!15}78.78/62.85 & \cellcolor{blue!15}45.18/83.60 & \cellcolor{blue!15}73.17/74.83 & \cellcolor{blue!15}58.00/79.25 & \cellcolor{blue!15}61.99/76.74 \\
& \cellcolor{blue!15}PRO-MSP-T & \cellcolor{blue!15}66.02/78.40 & \cellcolor{blue!15}\underline{50.63}/83.33 & \cellcolor{blue!15}\underline{61.53}/\underline{75.71} & \cellcolor{blue!15}38.79/86.49 & \cellcolor{blue!15}\textbf{60.08}/\textbf{79.94} & \cellcolor{blue!15}\textbf{54.50}/\textbf{80.96} & \cellcolor{blue!15}\textbf{55.26}/\textbf{80.80} \\
& \cellcolor{blue!15}PRO-ENT & \cellcolor{blue!15}63.48/\underline{78.63} & \cellcolor{blue!15}51.63/83.22 & \cellcolor{blue!15}73.51/68.52 & \cellcolor{blue!15}\textbf{34.88}/\textbf{89.64} & \cellcolor{blue!15}69.08/77.66 & \cellcolor{blue!15}56.98/80.31 & \cellcolor{blue!15}58.26/79.66 \\
& \cellcolor{blue!15}PRO-GEN & \cellcolor{blue!15}64.34/\textbf{78.64} & \cellcolor{blue!15}\textbf{50.36}/\textbf{83.50} & \cellcolor{blue!15}68.34/73.47 & \cellcolor{blue!15}\underline{37.96}/\underline{87.06} & \cellcolor{blue!15}65.40/\underline{79.14} & \cellcolor{blue!15}\underline{55.58}/\underline{80.83} & \cellcolor{blue!15}\underline{57.00}/\underline{80.44} \\
\midrule
\multirow{9}{*}{\rotatebox{90}{LRR-CARD-Deck}} & MSP\cite{hendrycks2016baseline}  & \textbf{57.77}/79.48 & 48.12/83.35 & \underline{59.53}/70.34 & 47.17/84.15 & 55.59/79.36 & 54.12/80.54 & 53.72/79.54 \\
& TempScaling\cite{guo2017calibration} & \textbf{57.77}/79.48 & 48.12/83.35 & \underline{59.53}/70.34 & 47.17/84.16 & 55.59/79.37 & 54.12/80.54 & 53.72/79.54 \\
& Entropy\cite{hendrycks2016baseline} & \textbf{57.77}/79.83 & 48.12/83.86 & \underline{59.53}/71.10 & 47.16/84.78 & 55.59/79.60 & 54.11/80.91 & \underline{53.71}/80.01 \\
& GEN\cite{liu2023gen} & 57.78/79.85 & 48.14/83.91 & 59.54/\underline{71.20} & 47.16/84.85 & \underline{55.58}/\underline{79.63} & 54.11/80.94 & 53.72/\underline{80.06} \\
& ODIN\cite{liang2017enhancing} & 58.73/77.26 & 49.40/81.51 & \textbf{47.80}/\textbf{80.06} & 44.83/84.80 & \textbf{55.17}/\textbf{80.57} & 54.31/79.24 & \textbf{51.71}/\textbf{80.57} \\
& \cellcolor{blue!15}PRO-MSP & \cellcolor{blue!15}58.54/79.69 & \cellcolor{blue!15}48.26/84.10 & \cellcolor{blue!15}76.64/64.89 & \cellcolor{blue!15}37.41/86.72 & \cellcolor{blue!15}65.57/77.44 & \cellcolor{blue!15}53.29/81.14 & \cellcolor{blue!15}56.62/79.00 \\
& \cellcolor{blue!15}PRO-MSP-T & \cellcolor{blue!15}58.52/79.69 & \cellcolor{blue!15}48.17/84.10 & \cellcolor{blue!15}76.64/64.90 & \cellcolor{blue!15}37.31/86.77 & \cellcolor{blue!15}65.56/77.45 & \cellcolor{blue!15}53.29/81.15 & \cellcolor{blue!15}56.58/79.01 \\
& \cellcolor{blue!15}PRO-ENT & \cellcolor{blue!15}58.46/\textbf{80.31} & \cellcolor{blue!15}\underline{46.73}/\underline{84.68} & \cellcolor{blue!15}72.91/66.93 & \cellcolor{blue!15}\textbf{32.80}/\textbf{88.13} & \cellcolor{blue!15}62.20/78.08 & \cellcolor{blue!15}\textbf{52.21}/\textbf{81.67} & \cellcolor{blue!15}54.22/79.97 \\
& \cellcolor{blue!15}PRO-GEN & \cellcolor{blue!15}58.78/\underline{80.26} & \cellcolor{blue!15}\textbf{45.98}/\textbf{84.74} & \cellcolor{blue!15}74.50/67.05 & \cellcolor{blue!15}\underline{33.23}/\underline{87.17} & \cellcolor{blue!15}63.48/77.84 & \cellcolor{blue!15}\underline{52.33}/\underline{81.61} & \cellcolor{blue!15}54.72/79.78 \\
\midrule
\multirow{9}{*}{\rotatebox{90}{AugMix-ResNeXt}} & MSP\cite{hendrycks2016baseline}  & \textbf{55.42}/80.12 & 51.20/81.69 & 51.33/80.55 & 53.92/78.48 & 67.49/73.44 & 55.39/79.32 & 55.79/78.93 \\
& TempScaling\cite{guo2017calibration} & 55.52/80.84 & 50.17/82.65 & 49.06/82.76 & 53.20/78.82 & 66.66/73.99 & 55.12/80.03 & 54.96/79.85 \\
& Entropy\cite{hendrycks2016baseline} & \underline{55.51}/81.08 & 50.63/82.95 & 50.24/83.29 & 53.46/78.84 & 67.41/73.95 & 55.17/80.18 & 55.40/80.05 \\
& GEN\cite{liu2023gen} & 57.81/81.02 & 51.04/83.17 & \underline{40.81}/\underline{86.32} & 52.54/78.21 & \textbf{66.01}/\underline{74.34} & 54.60/80.26 & \underline{53.80}/80.55 \\
& ODIN\cite{liang2017enhancing} & 66.33/77.34 & 62.21/78.68 & \textbf{19.21}/\textbf{95.70} & 78.33/66.72 & 74.37/72.57 & 62.93/76.59 & 60.56/77.93 \\
& \cellcolor{blue!15}PRO-MSP & \cellcolor{blue!15}58.48/80.08 & \cellcolor{blue!15}52.24/82.51 & \cellcolor{blue!15}61.60/79.00 & \cellcolor{blue!15}53.24/79.91 & \cellcolor{blue!15}73.07/72.05 & \cellcolor{blue!15}56.01/79.77 & \cellcolor{blue!15}59.11/78.89 \\
& \cellcolor{blue!15}PRO-MSP-T & \cellcolor{blue!15}58.33/80.93 & \cellcolor{blue!15}51.24/83.02 & \cellcolor{blue!15}41.50/85.99 & \cellcolor{blue!15}49.43/79.96 & \cellcolor{blue!15}67.62/74.15 & \cellcolor{blue!15}55.04/80.18 & \cellcolor{blue!15}53.86/80.70 \\
& \cellcolor{blue!15}PRO-ENT & \cellcolor{blue!15}56.56/\underline{81.17} & \cellcolor{blue!15}\underline{50.14}/\textbf{83.48} & \cellcolor{blue!15}52.33/83.06 & \cellcolor{blue!15}\textbf{40.87}/\textbf{85.80} & \cellcolor{blue!15}68.74/74.30 & \cellcolor{blue!15}\underline{54.51}/\textbf{80.75} & \cellcolor{blue!15}53.86/\textbf{81.43} \\
& \cellcolor{blue!15}PRO-GEN & \cellcolor{blue!15}57.58/\textbf{81.20} & \cellcolor{blue!15}\textbf{50.11}/\underline{83.44} & \cellcolor{blue!15}44.07/84.84 & \cellcolor{blue!15}\underline{45.74}/\underline{82.11} & \cellcolor{blue!15}\underline{66.44}/\textbf{74.57} & \cellcolor{blue!15}\textbf{53.59}/\underline{80.63} & \cellcolor{blue!15}\textbf{52.92}/\underline{81.13} \\
\bottomrule
\end{tabular}%
}
\caption{OOD detection performance on CIFAR-100 models across one default model and four robust models.}
\label{Table: supp-cifar100-robustmodels}
\end{table*}

\begin{figure}[h]
  \hspace{-0.80cm}
  \centering
  \includegraphics[width=0.60\textwidth]{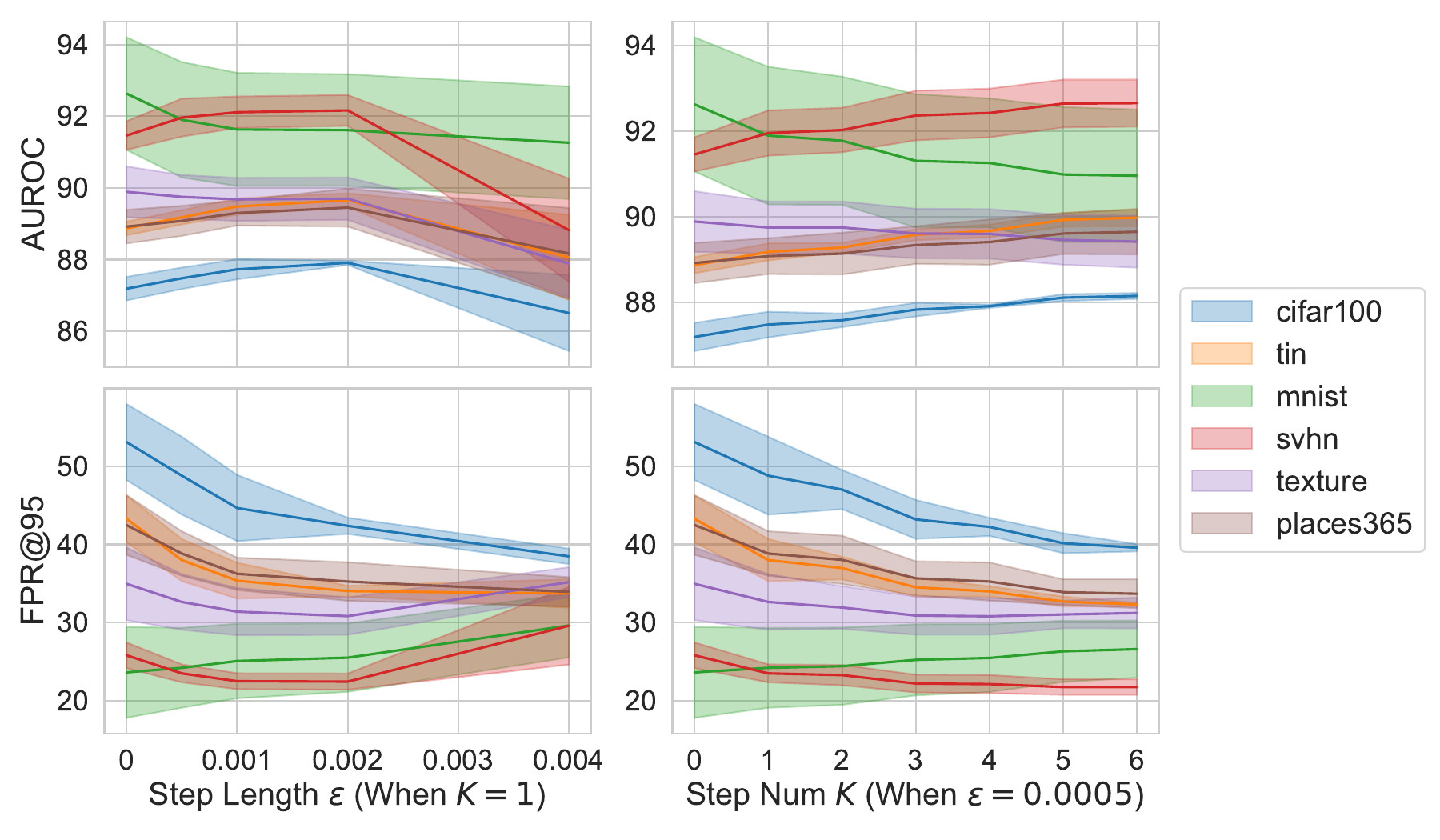}
  \centering
  \vspace{-0.2cm}
  \caption{Hyperparameter sensitivity analysis of PRO-MSP: Statistics are evaluated across three default CIFAR-10 models. obtained from independent training runs without applying robust training.
  }
  \label{figure: hyperparametersensitivity}
  \vspace{-0.7cm}
\end{figure}
\begin{table}[h]
\centering
\begin{tabular}{l|l|l}
\toprule
\textbf{Method} & \textbf{Dataset} & \textbf{Hyperparameters} \\
\midrule
\texttt{PRO-MSP}       & Cifar-10     & \{0.0003, 3\}            \\
$\{\epsilon,K\}$          & Cifar-100    & \{0.001, 5\}             \\
                         & ImageNet     & \{0.0005, 3\}            \\
\midrule
\texttt{PRO-MSP-T} & Cifar-10     & \{0.001, 5, 1000\}       \\
$\{\epsilon,K,T\}$     & Cifar-100    & \{0.001, 5, 10\}         \\
                         & ImageNet     & \{1.0e-05, 1, 10\}       \\

\midrule
\texttt{PRO-ENT}       & Cifar-10     & \{0.001, 1\}             \\
$\{\epsilon,K\}$           & Cifar-100    & \{0.0005, 7\}            \\
                         & ImageNet     & \{5.0e-05, 7\}           \\
\midrule
\texttt{PRO-GEN}        & Cifar-10     & \{0.1, 10, 0.001, 5\}    \\
$\{\gamma,M,\epsilon,K\}$       & Cifar-100    & \{0.01, 100, 0.0008, 5\} \\
                         & ImageNet     & \{0.1, 100, 0.0003, 1\}  \\
\bottomrule
\end{tabular}
\caption{Example hyperparameters of \ours}
\label{Table: hyperparameters}
\end{table}

\end{document}